\documentclass[11pt]{article}

\def \bN {\Bbb N}

\def \bR {\Bbb R}

\def \cD {{\cal D}}

\def \cI {{\cal I}}

\def \cN {{\cal N}}

\def \cR {{\cal R}}

\def \and {\, \mbox{\rm and}\, }

\def \supp {\,{\rm supp}\,}
\def \relu {\,{\rm ReLU}\,}

\def \diag {\,{\rm diag}\,}

\usepackage{mathrsfs}
\usepackage{verbatim}
\usepackage{amsmath}
\usepackage{amsfonts}
\usepackage{amssymb}
\usepackage{latexsym}
\usepackage[dvips]{graphicx}
\usepackage{subfigure}
\usepackage{epstopdf}
\usepackage{hyperref}
\newtheorem{theorem}{\bf Theorem}[section]
\newtheorem{lemma}[theorem]{\bf Lemma}

\newenvironment{proof}{\noindent{\em Proof:}}{\quad \hfill$\Box$\vspace{2ex}}

\newtheorem{definition}[theorem]{\bf Definition}

\def \bfa {{\mathbf a}}
\def \bfb {{\mathbf b}}
\def \bfc {{\mathbf c}}
\def \bfd {{\mathbf d}}
\def \bfe {{\mathbf e}}
\def \bff {{\mathbf f}}
\def \bfq {{\mathbf q}}

\def \bfw {{\mathbf w}}
\def \bfx {{\mathbf x}}
\def \bfy {{\mathbf y}}

\def \bfA {{\mathbf A}}
\def \bfB {{\mathbf B}}
\def \bfE {{\mathbf E}}
\def \bfI {{\mathbf I}}
\def \bfJ {{\mathbf J}}

\def \bfW {{\mathbf W}}
\def \bfT {{\mathbf T}}

\def \bbN {{\mathbb N}}
\def \bbR {{\mathbb R}}

\def \tbfb {{\widetilde{\mathbf b}}}
\def \tbfx {{\widetilde{\mathbf x}}}
\def \tbfy {{\widetilde{\mathbf y}}}

\def \barbfJ {{\bar {\mathbf J}}}
\def \barbfW {{\bar {\mathbf W}}}
\def \barbfb {{\bar {\mathbf b}}}

\def \ap {\,{\rm Ap}\,}

\def \vec {\,{\rm Vec}\,}
\def \relu {\,{\rm ReLU}\,}
\def \supp {\,{\rm supp}\,}
\def \diag {\,{\rm diag}\,}

\def \cD {{\cal D}}
\def \cI {{\cal I}}

\def \cN {{\cal N}}
\def \cR {{\cal R}}

\def \qquad {\quad\quad}

\def \bqquad {\quad\quad\quad\quad}

\makeatletter

\newcommand{\Rmnum}[1]{\expandafter\@slowromancap\romannumeral #1@}
\makeatother

\begin{document}
	
	\title{\bf Convergence Analysis of Deep Residual Networks}
	\author{Wentao Huang\thanks{School of Mathematics (Zhuhai), Sun Yat-sen University, Zhuhai, P.R. China. E-mail
			address: {\it huangwt55@mail2.sysu.edu.cn}.}
		\quad and \quad Haizhang Zhang\thanks{School of Mathematics (Zhuhai), Sun Yat-sen University, Zhuhai, P.R. China. E-mail
			address: {\it zhhaizh2@sysu.edu.cn}. Supported in part by National Natural Science Foundation of China under grant 11971490, and by Natural Science Foundation of Guangdong Province under grant 2018A030313841. Corresponding author.}}
	\date{}
	\maketitle
	\begin{abstract}
Various powerful deep neural network architectures have made great contribution to the exciting successes of deep learning in the past two decades. Among them, deep Residual Networks (ResNets) are of particular importance because they demonstrated great usefulness in computer vision by winning the first place in many deep learning competitions. Also, ResNets were the first class of neural networks in the development history of deep learning that are really deep. It is of mathematical interest and practical meaning to understand the convergence of deep ResNets. We aim at characterizing the convergence of deep ResNets as the depth tends to infinity in terms of the parameters of the networks. Toward this purpose, we first give a matrix-vector description of general deep neural networks with shortcut connections and formulate an explicit expression for the networks by using the notions of activation domains and activation matrices. The convergence is then reduced to the convergence of two series involving infinite products of non-square matrices. By studying the two series, we establish a sufficient condition for pointwise convergence of ResNets. Our result is able to give justification for the design of ResNets. We also conduct experiments on benchmark machine learning data to verify our results.

\medskip

\noindent{\bf Keywords:} deep learning, deep residual networks, ReLU networks, convolution neural networks, ResNets, convergence

\end{abstract}	
	\section{Introduction} \label{sec:introduction}
\setcounter{equation}{0}

In the past two decades, people have witnessed a series of major breakthroughs of artificial intelligence on a wide range of machine learning problems including face recognition, speech recognition, game intelligence, natural language processing, and autonomous navigation, \cite{Goodfellow,LeCun}. The breakthroughs are brought mainly by deep learning in which there are four major ingredients that contribute to the successes. The first three of them are the availability of vast amounts of training data, recent dramatic improvements in computing and storage power from computer engineering, and efficient numerical algorithms
such as the Stochastic Gradient Decent (SGD) algorithms, Adaptive Boosting (AdaBoost) algorithms, and the Expectation-Maximization algorithm (EM), etc. The last ingredient, which is generally considered to the most important one, is a class of deep neural network architectures, such as Convolutional Neural Networks (CNN), LongShort Time Memory (LSTM) networks, Recurrent Neural Networks (RNN), Generative Adversarial Networks (GAN), Deep Belief Networks (DBN), AlexNet \cite{AlexNet}, VGG-Net \cite{VGG}, GoogleLeNet/Inception \cite{GoogleLeNet}, and Residual Networks (ResNet) \cite{KaimingHe, KaimingHe2}.

Stimulated by the great achievements of deep learning, many mathematicians have been fascinated by DNNs, which is viewed as a nonlinear representation system of functions. Most mathematical studies of DNNs are focused on their approximation and expressive powers in representing different classes of functions, \cite{Adcock,I.Daubechies, Devore1, E,Elbrachter,Montanelli1,Montanelli2,Poggio,Shen1,Shen2,Shen3,Wang,Yarotsky,Zhou1}. We refer readers to the two recent surveys \cite{Devore1,Elbrachter} for a detailed introduction and discussion. Two pieces of work most related to our current study are \cite{ConvergenceDeepReluNet} and \cite{ConvergenceDeepCNN}, which first investigated the convergence of DNNs in terms of the weight matrices and bias vectors. Specifically, a sufficient condition for pointwise convergence was established for general DNNs with a fixed width and for CNNs with increasing width in \cite{ConvergenceDeepReluNet} and \cite{ConvergenceDeepCNN}, respectively.

Our current study targets at the Residual Network (ResNet). ResNet becomes well-known when it won the first place in the ILSVRC 2015 classification competition, in which it achieved a 3.57\% top-5 error on the ImageNet test set. It was also the first DNN architecture that is really deep. It contained 152 layers while the champions of ImageNet classification competition in 2013 and 2014, GoogleNet and VGG-Net contained 22 and 19 layers, respectively. ResNet also won the first place on ImageNet detection, ImageNet localization,
COCO detection, and COCO segmentation in ILSVRC \&
COCO 2015 competitions. ResNet hence deserves special attention for these reasons. Indeed, there have been many researches on the generalization ability, approximation power, and optimization of ResNets \cite{KaiXuan, He, AllenZhu, Zou, Frei, Lu, Lin, Qin, Zhang} with the hope of increasing the interpretability of ResNets. For instance, \cite{KaiXuan} partially justified the advantages of deep ResNets over deep FFNets in generalization abilities by comparing the kernel of deep ResNets with that of deep FFNets. It was proved in \cite{Lin} that a very deep ResNet with the ReLU activation function and stacked modules having one neuron per hidden layer can uniformly approximate any Lebesgue integrable function. And a generalization bound of ResNets was obtained in \cite{He}, which guarantees the performance of ResNet on unseen data.

Our series of studies \cite{ConvergenceDeepReluNet,ConvergenceDeepCNN} and the current one are devoted to the convergence of DNNs as the depth of networks tends to infinity. Convergence of a linear expansion of functions such as the Fourier and wavelet series has always been a fundamental problem in pure and applied mathematics. Understanding the convergence for the nonlinear system of DNNs is helpful to reveal the mathematical mystery of deep learning and will be useful in guiding the design of more DNN architectures. We have mentioned that \cite{ConvergenceDeepReluNet,ConvergenceDeepCNN} considered the convergence of a general DNN with a fixed width and CNN with increasing width, respectively. Our study is different and valuable in the following aspects:

\begin{enumerate}
    \item ResNets are special in containing residual blocks and shortcut connections. These two structures were not investigated in \cite{ConvergenceDeepReluNet,ConvergenceDeepCNN}.

    \item The two pieces of work \cite{ConvergenceDeepReluNet,ConvergenceDeepCNN} also studied DNNs with a single channel only, while ResNets and many other useful DNN architectures have multiple channels. Multiple channels will bring major mathematical difficulties in formulating an explicit expression of the function determined by a ResNet.

        \item Convergence of ResNets is of particular interest and can justify the original idea in designing the ResNet \cite{KaimingHe, KaimingHe2}. In some applications, the ResNet designed contained more than 1000 layers. The ability of ResNets in accommodating so many layers lies in the design of ResNet that for a network with so many layers to approximate a function of real interest, the function in deep layers must be close to the identity mapping \cite{KaimingHe2}. The result on convergence of ResNets to be established in this paper will be able to justify this design.

            \item We will carry out experiments based on well-known machine leaning data to verify the results of the paper.
\end{enumerate}

The rest of this paper is organized as follows. In Section \ref{sec:conv_with_zero_padding}, we shall review the definition and notation of convolution with zero padding and define the notation of vectorized version of convolutions. In Section \ref{sec:drn}, we shall describe the matrix form of ResNets with vectorized version of convolution in Section \ref{sec:conv_with_zero_padding}, which expresses a ResNet as a special form of DNNs with shortcut connections. In Section \ref{sec:dnn_with_sc}, we shall define the notion of convergence of DNNs with shortcut connections when new layers are paved to the existing network so that the depth is increasing to infinity. Then by introducing the notions of activation domains and activation matrices, we derive an explicit expression of deep ReLU network with shortcut connections. With the expression, we connect the convergence problem with the existence of two limits involving infinite products of matrices. Sufficient conditions for convergence of such infinite products of matrices are established in Section \ref{sec:convergence_of_dnn_with_sc}. As consequence, a sufficient condition for pointwise convergence of DNNs with shortcut connections is obtained in the same section. Based on those results, we shall establish convenient sufficient conditions for convergence of ResNets in Section \ref{sec:convergence_of_drn}. Finally, we shall justify our theoretical results with numerical experiments on well-known machine learning data.

\section{Convolution with Zero Padding} \label{sec:conv_with_zero_padding}
\setcounter{equation}{0}

	Before reviewing the structure of ResNets, we first explore the matrix form of the convolution operation with zero padding which will be used to formulate a matrix form of ResNets in the next section.
	
	\subsection{Convolution of Vectors with Zero Padding in Single Channel}
		
		We will first examine the operation of convolution of vectors with zero padding in single channle.
		
		Let $f\in\bN$ and assume $\bfw := (w_0, w_1, \cdots, w_{2f}) \in \bbR^{2f+1}$ is filter mask with size $2f+1$ and $\bfx := (x_1, x_2, \cdots, x_d) \in \bbR^{d}$. Then the convolution with zero padding of $\bfx \ast \bfw$ outputs a vector $\bfy := (y_1, y_2, \cdots, y_d) \in \bbR^{d}$, which preserves the dimension of $\bfx$, defined by
		\begin{equation*} \label{eq:convolution 1-d single channel}
			y_i = \sum_{j=\max(1-i, -f)}^{\min(d-i, f)} x_{i+j}w_{f-j},\ 1 \leq i \leq d.
		\end{equation*}
	It is convenient to express the convolution $\bfx \ast \bfw$ as multiplication of $
		\bfx$ with a matrix corresponding to $\bfw$. To this end, we define a Toeplitz type matrix $\bfT(\bfw) \in \bbR^{d \times d}$ by
		\begin{equation*} \label{eq:T(w) element}
			\bfT(\bfw)_{i,j}:= \left\{
				\begin{aligned}
					w_{f-j+i} &,\ -f \leq j - i \leq f \\
					0 &,\ otherwise
				\end{aligned}
			\right.
		\end{equation*}
		and rewrite
		\begin{equation*} \label{eq:matrix form of convolution 1-d}
			\bfy = \bfx \ast \bfw = \bfT(\bfw)\vec(\bfx).
		\end{equation*}
		Here, $\vec(\bfx)$ is a vectorized version of $\bfx$ defined as follows.
		\begin{definition} {\bf (The vec operator)} \label{def:vec}
			Assume $\bfx := (\bfx_1, \bfx_2, \cdots, \bfx_n) \in \bbR^{m \times n}$ and $\bfx_{j} := \\ (x_{1,j}, \cdots, x_{m,j})^T \in  \bbR^{m},\ 1\leq j\leq n$. Then $\vec(x) \in \bbR^{mn}$ is defined by
			\begin{equation*}
				\vec(\bfx):=
					\begin{bmatrix}
						\bfx_1 \\
						\bfx_2 \\
						\vdots \\
						\bfx_n
					\end{bmatrix}.
			\end{equation*}
		\end{definition}

		We shall assume $f+1 \leq d$ since the size of the filter mask is much less than the size of data in real applications. Under this assumption, the matrix $\bfT(\bfw)$ can be expressed as
		\begin{equation} \label{eq:T(w) matrix}
			\bfT(\bfw):=
				\begin{bmatrix}
					w_f & w_{f-1} & \cdots & w_0 & 0 & \cdots & 0 \\
					w_{f+1} & \ddots & \ddots & \ddots & \ddots & \ddots & \vdots \\
					\vdots & \ddots & \ddots & \ddots & \ddots & \ddots & \vdots \\
					w_{2f} & \ddots & \ddots & \ddots & \ddots & \ddots & w_{0} \\
					0 & \ddots & \ddots & \ddots & \ddots & \ddots & \vdots \\
					\vdots & \ddots & \ddots & \ddots & \ddots & \ddots & \vdots \\
					0 & \cdots & \cdots & w_{2f} & \cdots & \cdots & w_{f} \\	
				\end{bmatrix}\in\bbR^{d\times d}.
		\end{equation}
		Note that the entries of $\bfT(\bfw)$ along each of the diagonal and sub-diagonals are constant. 
		
	\subsection{Convolution of Matrices with Zero Padding in Single Channel}
	
		We shall further formulate convolution of matrices with zero padding in single channel.
		
		Assume $\bfw := (\bfw_{0}^T, \bfw_{1}^T, \cdots, \bfw_{2f}^T)^T \in \bbR^{(2f+1)\times(2f+1)}$, where $\bfw_{i} := (w_{i,0}, w_{i,1}, \cdots, w_{i, 2f}) \in \bbR^{2f+1},\ 0\leq i\leq 2f+1$, and $\bfx := (\bfx_1^T, \bfx_2^T, \cdots, \bfx_d^T)^T \in \bbR^{d \times d}$, where $\bfx_{i} := (x_{i,1}, x_{i,2}, \cdots, x_{i,d}) \in \bbR^{d}, 1\leq i\leq d$. Then the convolution $\bfx \ast \bfw$ with zero padding outputs a matrix $\bfy := (\bfy_1^T, \bfy_2^T, \cdots, \bfy_d^T)^T \in \bbR^{d \times d}$ with $\bfy_{i} := (y_{i,1}, y_{i,2}, \cdots, y_{i,d}) \in \bbR^{d}, 1\leq i\leq d$, defined by
		\begin{equation*} \label{eq:convolution 2-d single channel}
			y_{i,j} = \sum_{k_1=\max(1-i, -f)}^{\min(d-i, f)}\sum_{k_2=\max(1-j, -f)}^{\min(d-j, f)}x_{i+k_1, j+k_2}w_{f-k_1, f-k_2},\ 1\leq i,j\leq d.
		\end{equation*}
		Notice that $\vec(\bfx^T)=(\bfx_1, \bfx_2, \cdots, \bfx_d)^T \in\bbR^{d^{2}}$ and $\vec(\bfy^T)=(\bfy_1, \bfy_2, \cdots, \bfy_d)^T\in\bbR^{d^{2}}$ are column vectors and thus we could rewrite $\bfx \ast \bfw$ in a matrix-vector multiplication form by
		\begin{equation} \label{eq:T(w) 2-d single channel}
			\begin{aligned}
				\vec(\bfy^T)
				&= \vec((\bfx \ast \bfw)^T) \\
				&= \bfT(\bfw)\vec(\bfx^T) \\
				&=
				\begin{bmatrix}
					\bfT(\bfw_f) & \bfT(\bfw_{f-1}) & \cdots & \bfT(\bfw_0) & 0 & \cdots & 0 \\
					\bfT(\bfw_{f+1}) & \ddots & \ddots & \ddots & \ddots & \ddots & \vdots \\
					\vdots & \ddots & \ddots & \ddots & \ddots & \ddots & \vdots \\
					\bfT(\bfw_{2f}) & \ddots & \ddots & \ddots & \ddots & \ddots & \bfT(\bfw_{0}) \\
					0 & \ddots & \ddots & \ddots & \ddots & \ddots & \vdots \\
					\vdots & \ddots & \ddots & \ddots & \ddots & \ddots & \vdots \\
					0 & \cdots & \cdots & \bfT(\bfw_{2f}) & \cdots & \cdots & \bfT(\bfw_{f}) \\	
				\end{bmatrix}
				\vec(\bfx^T).
			\end{aligned}
		\end{equation}
		Here, the matrix $\bfT(\bfw)\in\bbR^{d^2\times d^2}$ is a Toeplitz type matrix with block Toeplitz structure and $\bfT(\bfw_{i})\in\bbR^{d\times d}$ are defined as in (\ref{eq:matrix form of convolution 1-d}).
		
	\subsection{Convolution of Matrices with Zero Padding in Multiple Channels}
		
		With the above preparations, we could now give the matrix-vector multiplication form of covolution with zero padding for matrices in multiple channels.
		
Let $c_{in}$ and $c_{out}$ denote the number of input and output channels, respectively. Assume $\bfx:=(\bfx_1, \cdots, \bfx_{c_{in}})\in\bbR^{d\times d\times c_{in}}$ with each $\bfx_i\in\bR^{d\times d}$, $\bfw:=(\bfw_{1}, \cdots, \bfw_{c_{out}})$ where $\bfw_{i}:=(\bfw_{i,1}, \bfw_{i,2}, \cdots,\bfw_{i,c_{in}})\in\bbR^{(2f+1)\times(2f+1)\times c_{in}},\ 1\leq i\leq c_{out}$ with each $\bfw_{i,j}\in\bbR^{(2f+1)\times(2f+1)}$, and $\bfy:=(\bfy_1, \bfy_2, \cdots, \bfy_{c_{out}})\in\bbR^{d\times d\times c_{out}}$ with each $\bfy_i\in\bR^{d\times d}$. For convolution of matrices in multiple channels, we define
		\begin{equation*} \label{eq:convolution 2-d multiple channel}
			\vec(\bfy_{i}^T)=\sum_{j=1}^{c_{in}}\bfT(\bfw_{i,j})\vec(\bfx_{j}^T),\ 1\leq i\leq c_{out}
		\end{equation*}
		and
		\begin{equation*}
			\begin{aligned}
				\vec(\bfx)&:=(\vec(\bfx_1^T)^T, \vec(\bfx_2^T)^T, \cdots, \vec(\bfx_{c_{in}}^T)^T)^T\in\bbR^{d^2c_{in}},\\
				\vec(\bfy)&:=(\vec(\bfy_1^T)^T, \vec(\bfy_2^T)^T, \cdots, \vec(\bfy_{c_{out}}^T)^T)^T\in\bbR^{d^2c_{out}}.
			\end{aligned}
		\end{equation*}
		Then the matrix-vector multiplication form of convolution of matrices with zero padding in multiple channels is given as follows:
		\begin{equation} \label{eq:T(w) 2-d multiple channels}
			\begin{aligned}
				\vec(\bfy)&=\bfW\vec(\bfx),\\
				\mbox{where }\bfW&=\bfT(\bfw)=
				\begin{bmatrix}
					\bfT(\bfw_{1,1}) & \bfT(\bfw_{1,2}) & \cdots & \bfT(\bfw_{1,c_{in}}) \\
					\bfT(\bfw_{2,1}) & \bfT(\bfw_{2,2}) & \cdots & \bfT(\bfw_{2,c_{in}}) \\
					\vdots & \vdots & \ddots & \vdots \\
					\bfT(\bfw_{c_{out},1}) & \bfT(\bfw_{c_{out},2}) & \cdots & \bfT(\bfw_{c_{out},c_{in}})
				\end{bmatrix}.
			\end{aligned}
		\end{equation}
		Notice that $\bfW\in\bbR^{d^{2}c_{out}\times d^{2}c_{in}}$  is a block matrix and every block in $\bfW$ is defined by (\ref{eq:T(w) 2-d single channel}).
		
		In conclusion, we have formulated a matrix form for convolution of matrices with zero padding in multi-channels in ResNets. It will be helpful for us to derive an explicit expression for the function determined by a ResNet.

\section{Deep Residual Networks} \label{sec:drn}
\setcounter{equation}{0}
	
	In order to study the convergence of ResNets, we consider in this section a matrix form of ResNets. 
Let us start with the pure convolution structure with shortcut connection of a ReLU neural network. Let $\sigma$ denote the {\bf ReLU} activation function
	\begin{equation*} \label{eq:relu definition}
		\sigma(x):=\max(x,0),\ x\in\bbR.
	\end{equation*}
	A ResNet as proposed in \cite{KaimingHe} with $n$ building blocks may be illustrated as follows:
	\begin{equation} \label{eq:structure of ResNet}
		\begin{aligned}
			\bfx \in [0,1]^{d\times d\times c_{in}}
			& \xrightarrow[\sigma]{\bfw_{s}, \bfb_{s}} \bfx^{(0)}
			& \xrightarrow{\cR^{(1)}_{\bfc^{(1)},\bff^{(1)}, q_{1}}} \bfx^{(1)} \xrightarrow{} \cdots \xrightarrow{} \xrightarrow{\cR^{(n)}_{\bfc^{(n)},\bff^{(n)}, q_{n}}} \bfx^{(n)}
			& \xrightarrow[\bfW_{o},\bfb_{o}]{\ap} \bfy \in \bbR^{c_{out}}.
			\\
			\mbox{input}\quad
			& \quad\mbox{sampling}
			& \mbox{residual blocks}\bqquad
			& \quad\mbox{output}
		\end{aligned}
	\end{equation}

We shall explain the above structure in details. In the sampling layer, it holds
\begin{equation} \label{samplinglayerresnt}
			\bfx^{(0)}=\sigma(\bfx \ast \bfw_{s} + \bfb_{s}),
	\end{equation}
where $\bfw_s$ and $\bfb_s$ denote the filter mask and bias vector of the sampling layer, respectively. And in the output layer, it holds
\begin{equation} \label{outputlayerresnt}
			\vec(\bfy)=\bfW_{o}\ap(\vec(\bfx^{(n)}))+\bfb_{o},
	\end{equation}
where $\bfW_o$ and $\bfb_o$ are the weight matrix and bias vector of the output layer, and $\ap$ denotes an global average pooling operation. It has been revealed in \cite{IdentityMatters} that max-pooling is not necessary in an all-convolutional residual network. Thus we shall only consider the global average pooling.

The most complicate structure lies in the residual blocks $\bfx^{(k-1)}\to \bfx^{(k)}$, $1\le k\le n$. It is represented by a nonlinear operator
\begin{equation} \label{eq:convolutional resnet}
\bfx^{(k)}=\cR^{(k)}_{\bfc^{(k)},\bff^{(k)}, q_{k}}(\bfx^{(k-1)}),
\end{equation}
where $\bfc^{(k)}:=(c^{(k)}_{0}, \cdots, c^{(k)}_{q_{k}})\in\bbN_{+}^{q_{k}+1}$, $\bff^{(k)}:=(f^{(k)}_{1}, \cdots, f^{(k)}_{q_{k}})\in\bbN_{+}^{q_{k}}$, and $q_{k}\in\bN$ denote the numbers of channels at each layer, the sizes of filter masks, and the depth of $k$-th residual block, respectively. Note that in most applications, $c_{in}\leq c^{(k)}_{m}$ for all $1\leq k\leq n,\ 1\leq m\leq q_{n}$. We shall make this assumption throughout the paper. Also, $c^{(k)}_{q_{k}}=c^{(k)}_{0}$ so that the dimensions are match in vector additions. The nonlinear operator $\cR^{(k)}_{\bfc^{(k)},\bff^{(k)}, q_{k}}$ may be illustrated by
	\begin{equation} \label{eq:structure of k-th residual block}
		\begin{aligned}
			\bfx^{(k-1)}=\bfx^{(k-1)}_{0} 
			& \xrightarrow[\sigma]{\bfw^{(k)}_{1}, \bfb^{(k)}_{1}} \bfx^{(k-1)}_{1}
			& \xrightarrow[\sigma]{\bfw^{(k)}_{2}, \bfb^{(k)}_{2}} \bfx^{(k-1)}_{2}
			& \xrightarrow{} \cdots \xrightarrow{}
			& \xrightarrow[\sigma]{\bfw^{(k)}_{q_k}, \bfb^{(k)}_{q_k}, +\bfx^{(k-1)}} \bfx^{(k-1)}_{q_{k}}=\bfx^{(k)}
			\\
			\mbox{input}\qquad
			& \qquad\mbox{1st layer}
			& \mbox{2nd layer}\qquad
			&
			& \mbox{$q_{k}$-th layer}\bqquad
		\end{aligned}
	\end{equation}
and mathematically written as
\begin{equation} \label{eq:convolutional residual part}
		\begin{aligned}
			\bfx^{(k-1)}_{i}&:=\sigma(\bfx^{(k-1)}_{i-1} \ast \bfw^{(k)}_{i} + \bfb^{(k)}_{i})\in\bbR^{d\times d\times c^{(k)}_{i}},\ 1\leq i\leq q_{k}-1,\\
			\bfx^{(k)}=\bfx^{(k-1)}_{q_{k}}&:=\sigma(\bfx^{(k-1)}_{q_{k}-1} \ast \bfw^{(k)}_{q_{k}} + \bfx^{(k-1)} + \bfb^{(k)}_{i})\in\bbR^{d\times d\times c^{(k)}_{q_{k}}}.
		\end{aligned}
	\end{equation}
In the above, $\bfw^{(k)}_{i}:=(\bfw^{(k)}_{i,1},\bfw^{(k)}_{i,2},\cdots,\bfw^{(k)}_{i,c^{(k)}_{i}})\in\bbR^{(2f^{(k)}_{i}+1)\times(2f^{(k)}_{i}+1)\times c^{(k)}_{i}}$ and $\bfb^{(k)}_{i}:=(b^{(k)}_{i,1}\bfE_{d}, \cdots, \\b^{(k)}_{i,c^{(k)}_{i}}\bfE_{d})$, where $\bfE_{d}$ denotes the $d\times d$ all-ones matrix and $b^{(k)}_{i,j}\in\bbR$ ($1\le j\le c^{(k)}_{i}$), are the filter mask and the bias vector at the $i$-th layer of the $k$-th residual block, respectively.
	
	By the matrix languages introduced in Section \ref{sec:conv_with_zero_padding}, we can express all the convolution operations above as matrix-vector multiplications. Let $\tbfx:=\vec(\bfx)\in\bbR^{d^{2}c_{in}}$, $\tbfx^{(k)}:=\vec(\bfx^{(k)})\in\bbR^{d^{2}c^{k}_{0}}$, and $\tbfy:=\vec(\bfy)=\bfy\in\bbR^{d^{2}c_{out}}$. Equations (\ref{samplinglayerresnt})-(\ref{eq:convolutional resnet}), and (\ref{eq:convolutional resnet}) can then be rewritten as
\begin{equation} \label{eq:matrix resnet}
		\begin{aligned}
			\bfW_{s}&:=\bfT(\bfw_{s}),	\quad		\tbfb_{s}:=\vec(\bfb_{s}),\\
			\tbfx^{(0)}&:=\sigma(\bfW_{s}\tbfx + \tbfb_{s}),\quad			
			\tbfy:=\bfW_{fc}\ap(\tbfx^{(n)})+\bfb_{fc},\\
\tbfx^{(k)}&:=\cN^{(k)}_{\bfc^{(k)}, \bff^{(k)}, q_{k}}(\tbfx^{(k-1)}),\ 1\leq k\leq n,\\
		\end{aligned}
	\end{equation}
and
	\begin{equation} \label{eq:matrix residual part}
		\begin{aligned}
			\bfW^{(k)}_{i}&:=\bfT(\bfw^{(k)}_{i}),\quad
			\tbfb^{(k)}_{i}:=\vec(\bfb^{(k)}_{i}),\\
			\tbfx^{(k-1)}_{i}&:=\sigma(\bfW^{(k)}_{i}\tbfx^{(k-1)}_{i-1} + \tbfb^{(k)}_{i})\in\bbR^{d^{2}c^{(k)}_{i}},\ 1\leq i\leq q_{k}-1,\\
			\tbfx^{(k)}&=\cN^{(k)}_{\bfc^{(k)},\bff^{(k)}, q_{k}}(\tbfx^{(k-1)})=\tbfx^{(k-1)}_{q_{k}}=\sigma(\bfW^{(k)}_{q_{k}}\tbfx^{(k-1)}_{q_{k}-1} + \tbfx^{(k-1)} + \tbfb^{(k)}_{i})\in\bbR^{d^{2}c^{(k)}_{q_{k}}}.
		\end{aligned}
	\end{equation}

	With the above matrix forms, conditions on the filter masks and bias vectors that ensure the convergence of network (\ref{eq:structure of ResNet}) as the number of residual blocks $n$ tends to infinity, can then be reformulated as conditions on the weight matrices and bias vectors. Convergence of deep ReLU neural networks and deep ReLU CNNs was recently studied in \cite{ConvergenceDeepReluNet,ConvergenceDeepCNN}. The matrix form for ResNets differs from those for DNNs and CNNs considered there in that the ResNets have shortcut connections while the DNNs in \cite{ConvergenceDeepReluNet} or CNNs in \cite{ConvergenceDeepCNN} do not. Also, the ResNets we consider have multiple-channels while the networks in \cite{ConvergenceDeepReluNet,ConvergenceDeepCNN} both have a single channel. These differences cause major difficulties that will be clear as we proceed with the analysis in the next sections. As a consequences, the results in \cite{ConvergenceDeepReluNet,ConvergenceDeepCNN} could not be directly applied to the current setting.
\section{Deep Neural Networks with Shortcut connections} \label{sec:dnn_with_sc}
\setcounter{equation}{0}

Since a ResNet (\ref{eq:structure of ResNet}) is a special deep ReLU neural network with shortcut connections, we shall first study convergence of general fully connected feed-forward neural networks with shortcut connections. The result obtained will then be applied to establish the convergence of ResNets. Toward this purpose, we shall describe general fully connected neural networks with shortcut connections and formulate their convergence as convergence of infinite products of non-square matrices.
	
Let $d_{in},d_{out},d_{res}$ be the dimension of the input space, output space, and the domain of residual blocks, respectively. We shall always assume $d_{in}\leq d_{res}$ as this is the case in real applications. The weight matrix $\bfW_{s}$ and the bias vector $\bfb_{s}$ of the sampling layer satisfy $\bfW_{s}\in\bbR^{d_{res}\times d_{in}}$ and $\bfb_{s}\in\bbR^{d_{res}}$. For each $1\leq k\leq n$, let $\bfW^{(k)}:=(\bfW^{(k)}_{1},\cdots, \bfW^{(k)}_{q_{k}})$ and $\bfb^{(k)}:=(\bfb^{(k)}_{1},\cdots,\bfb^{(k)}_{q_{k}})$ denote the weight matrices and bias vectors in the $k$-th residual layer, and $\bfc^{(k)}=(c^{(k)}_{0},\cdots,c^{(k)}_{q_{k}})\in\bbN^{q_{k}+1}_{+}$ denotes the dimensions of outputs in the $k$-th residual layer. That is, $\bfb^{(k)}_{i}\in\bbR^{c^{(k)}_{i}}$ for $1\leq i\leq q_{k},\ 1\leq k\leq n$, and $\bfW^{(k)}_{i}\in\bbR^{c^{(k)}_{i}\times c^{(k)}_{i-1}}$ for $1\leq i\leq q_{k},\ 1\leq k\leq n$. Note that $c^{(k)}_{0}=c^{(k)}_{q_{k}}=d_{res}$ for $1\leq i\leq q_{k},\ 1\leq k\leq n$. The weight matrix $\bfW_{o}$ and the bias vector $\bfb_{o}$ of the output layer satisfy $\bfW_{o}\in\bbR^{d_{out}\times d_{res}}$ and $\bfb_{o}\in\bbR^{d_{out}}$. And $\cN^{(k)}_{\bfc^{(k)}, q_{k}}$  denotes the nonlinear operator determined by the $k$-th residual blocks,  $1\leq k\leq n$.

	The structure of such a deep neural network with shortcut connections and the ReLU activation function $\sigma$ is illustrated as follows:
	\begin{equation} \label{eq:structure of DNN with sc}
		\begin{aligned}
			\bfx \in [0,1]^{d_{in}}
			& \xrightarrow[\sigma]{\bfW_{s}, \bfb_{s}} \bfx^{(0)}
			& \xrightarrow{\cN^{(1)}_{\bfc^{(1)}, q_{1}}} \bfx^{(1)} \xrightarrow{} \cdots \xrightarrow{} \xrightarrow{\cN^{(n)}_{\bfc^{(n)}, q_{n}}} \bfx^{(n)}
			& \xrightarrow{\bfW_{o}, \bfb_{o}} \bfy \in \bbR^{d_{out}}.
			\\
			\mbox{input}\quad
			& \quad\mbox{sampling}
			& \mbox{residual blocks}\bqquad
			& \quad\mbox{output}
		\end{aligned}
	\end{equation}
where the $k$th residual block is illustrated by
	\begin{equation} \label{eq:structure of residual block in DNN}
		\begin{aligned}
			\bfx^{(k-1)}\in\bbR^{d_{res}}
			& \xrightarrow[\sigma]{\bfW^{(k)}_{1}, \bfb^{(k)}_{1}} \bfx^{(k-1)}_{1}
			& \xrightarrow{} \cdots \xrightarrow{}
			& \xrightarrow[\sigma]{\bfW^{(k)}_{q_k}, \bfb^{(k)}_{q_k}, +\bfx^{(k-1)}} \bfx^{(k)}\in\bbR^{d_{res}},\ 1\leq k\leq n.
			\\
			\mbox{input}\qquad
			& \qquad\mbox{1st layer}
			&
			& \bqquad\mbox{$q_{k}$-th layer}
		\end{aligned}
	\end{equation}
	Here
	\begin{equation} \label{eq:resnet formulate}
			\bfx^{(0)}:=\sigma(\bfW_{s}\bfx+\bfb_{s}),\ \bfy:=\bfW_{o}\bfx^{(n)}+\bfb_{o},\quad
			\bfx^{(k)}:=\cN^{(k)}_{\bfc^{(k)}, q_{k}}(\bfx^{(k-1)}),\ 1\leq k\leq n,\\
	\end{equation}
	and $\cN^{(k)}_{\bfc^{(k)}, q_{k}}$ is given by
	\begin{equation} \label{eq:residual formulate}
		\begin{aligned}
			\bfx^{(k-1)}_{0}&:=\bfx^{(k-1)},\
			\bfx^{(k-1)}_{i}:=\sigma(\bfW^{(k)}_{i}x^{(k-1)}_{i-1}+\bfb^{(k)}_{i}),\ 1\leq i\leq q_{k}-1,\\
			\bfx^{(k)}=\cN^{(k)}_{\bfc^{(k)}, q_{k}}(\bfx^{(k-1)})&:=\sigma(\bfW^{(k)}_{q_{k}}\bfx^{(k-1)}_{q_{k}-1}+\bfx^{(k-1)}+\bfb^{(k)}_{q_{k}-1}).
		\end{aligned}	
	\end{equation}
Note that the sampling layer could be treated as a single layer residual block with the weight matrix $\bfW^{(0)}_{1}:=\begin{bmatrix}
		\bfW_{s}& {\bf 0}
	\end{bmatrix}
	-\bfI_{d_{res}}
	\in\bbR^{d_{res}\times d_{res}}$, where $\bfI_{d_{res}}$ denotes the $d_{res}\times d_{res}$ identity matrix and bias vector $\bfb^{(0)}_{1}:=\bfb_{s}\in\bbR^{d_{res}}$. For convenience, we may adjust the input $\bfx$ to $\bfx^{(-1)}:=\begin{bmatrix}
		\bfx\\
		{\bf 0}
	\end{bmatrix}\in\bbR^{d_{res}}$, and rewrite equations (\ref{eq:resnet formulate}) and (\ref{eq:residual formulate}) as:
	\begin{equation} \label{eq:resnet formulate adjust}
			\bfx^{(k)}=\cN^{(k)}_{\bfc^{(k)}, q_{k}}(\bfx^{(k-1)}),\ 0\leq k\leq n
	\end{equation}
	and
	\begin{equation} \label{eq:residual formulate adjust}
		\begin{aligned}
			\bfx^{(-1)}_{0}&:=\bfx^{(-1)},
			\bfx^{(k-1)}_{0}:=\bfx^{(k-1)},\\
			\bfx^{(k-1)}_{i}&:=\sigma(\bfW^{(k)}_{i}x^{(k-1)}_{i-1}+\bfb^{(k)}_{i}),\ 1\leq i\leq q_{k}-1,\\
			\bfx^{(k)}&:=\sigma(\bfW^{(k)}_{q_{k}}\bfx^{(k-1)}_{q_{k}-1}+\bfx^{(k-1)}+\bfb^{(k)}_{q_{k}})
		\end{aligned}
		\qquad 0\leq k\leq n,
	\end{equation}
	where $q_{0}=1$, $\bfc^{(0)}:=(c^{(0)}_{0}, c^{(0)}_{1})$ with $c^{(0)}_{0}=d_{in}, c^{(0)}_{1}=d_{res}$, $\bfW^{(0)}:=(\bfW^{(0)}_{1})$, and $\bfb^{(0)}:=(\bfb^{(0)}_{1})$.

To formulate an expression for the continuous function determined by the deep neural network (\ref{eq:structure of DNN with sc}) with shortcut connections, we recall the compact notation for consecutive compositions of functions which was used in \cite{ConvergenceDeepReluNet}.

	\begin{definition} {\bf (Consecutive composition)}
		Let $f_1$, $f_2$, $\cdots$, $f_n$ be a finite sequence of functions such that the range of $f_i$ is contained in the domain of $f_{i+1}$, $1\leq i\leq n-1$, the consecutive composition of $\{f_{i}\}_{i=1}^{n}$ is defined to be a function
		\begin{equation*}
			\bigodot_{i=1}^{n}f_{i}:=f_{n}\circ f_{n-1}\circ \cdots \circ f_{2} \circ f_{1},
		\end{equation*}
		whose domain is that of $f_1$. For convenience, we also make the convention $(\bigodot_{i=1}^{n}f_{i})(\bfx)=\bfx$ if $n<1$.
	\end{definition}
	
	Using the above notation, equation (\ref{eq:resnet formulate adjust}) and (\ref{eq:residual formulate adjust}) may be rewritten as
	\begin{equation} \label{eq:resnet formulate comp}
			\bfx^{(k)}=\left(\bigodot_{m=0}^{k}\cN^{(m)}_{\bfc^{(m)}, q_{m}}\right)(\bfx^{(-1)}),\ 0\leq k\leq n,
	\end{equation}
	and
	\begin{equation} \label{eq:residual formulate comp}
		\begin{aligned}
			\bfx^{(k-1)}_{i}&=\left(\bigodot_{m'=1}^{i}\sigma(\bfW^{(k)}_{m'}\cdot+\bfb^{(k)}_{m'})\right)(\bfx^{(k-1)}_{0}),\ 0\leq i\leq q_{k}-1,\\
			\bfx^{(k)}&=\sigma(\bfW^{(k)}_{q_{k}}\bfx^{(k-1)}_{q_{k}-1}+\bfx^{(k-1)}+\bfb^{(k)}_{q_{k}-1}),
		\end{aligned}
		\qquad 0\leq k\leq n.
	\end{equation}
And we have
	\begin{equation} \label{eq:residual block formulate}
		\cN^{(k)}_{\bfc^{(k)},q_k}(\bfx^{(k-1)}):=\left(\sigma\left(\bfW^{(k)}_{q_k}\left(\bigodot_{i=1}^{q_{k}}\sigma(\bfW^{(k)}_{i}\cdot+\bfb^{(k)}_{i})\right)+\cdot+\bfb^{(k)}_{q_{k}}\right)\right)(\bfx^{(k-1)}),\ 0\leq k\leq n.
	\end{equation}
	We are concerned with the convergence of the above functions determined by the deep neural network as $n$ tends to infinity. Since the output layer is a linear function of $x^{(n)}$ and it will not affect the convergence of the network. We are hence concerned with the convergence of the deep neural work defined by
	\begin{equation} \label{eq:convergence part D-resnet}
		\cN_{\bfc, \bfq, n}(\bfx):=\left(\bigodot_{k=0}^{n}\cN^{(k)}_{\bfc^{(k)}, q_{k}}\right)(\bfx),\ \bfx\in [0,1]^{d_{res}}
	\end{equation}
	as $n$ tends to infinity. Note that here $\bfc:=(\bfc^{(0)},\bfc^{(1)},\cdots,\bfc^{(n)})$, $\bfq:=(q_{0},q_{1},\cdots,q_{n})$ and $\cN_{\bfc, \bfq, n}$ is a function from $[0,1]^{d_{res}}$ to $\bbR^{d_{res}}$.

	We introduce an algebraic formulation of a deep ReLU network with shortcut connections by adapting the notions of activation domains and activation matrices introduced in \cite{ConvergenceDeepReluNet}. For each $m\in\bbN_{+}$, we define the set of {\bf activation matrices} by
	\begin{equation*}
		\begin{aligned}
			\cD_{m}&:=\{\diag(a_{1}, \cdots, a_{m}):a_{i}\in\{0,1\},1\leq i\leq m\}.\\
		\end{aligned}
	\end{equation*}
	The support of an activation matrix $J\in\cD_{m}$ is defined by
	\begin{equation*}
		\supp J:=\{k:\ J_{k,k}=1,\ 1\leq k\leq m\}.
	\end{equation*}
	
	\begin{definition} {\bf (Activation domains of one layer network)} \label{def:Activation domains of one layer network}
		For a weight matrix $\bfW\in\bbR^{d\times d^{'}}$ and a bias vector $\bfb\in\bbR^{d}$, the activation domain of $\sigma(\bfW\bfx+\bfb)$ with respect to a diagonal matrix $J\in\cD_{m}$ is
		\begin{equation*}
			D_{J, \bfW, \bfb}:=\{x\in\bbR^{d^{'}}:(\bfW\bfx+\bfb)_{j}>0,\ for\ j\in\supp J\ and\ (\bfW\bfx+\bfb)_{j}\leq 0\ for\ j\notin\supp J\}.
		\end{equation*}
	\end{definition}

	
We need to extend the definition of activation domains of neural network with one and multiple residual blocks.
	\begin{definition} {\bf (Activation domains of one-residual-block network)} \label{def:Activation domains of residual block}
		For
		\begin{equation*}
			\bfW^{(0)}:=(\bfW^{(0)}_{1}, \cdots, \bfW^{(0)}_{q_{0}})\in\prod_{m=1}^{q_{0}}\bbR^{c^{(0)}_{m}\times c^{(0)}_{m-1}},\quad and\quad \bfb^{(0)}:=(\bfb^{(0)}_{1}, \cdots, \bfb^{(0)}_{q_{n}})\in\prod_{m=1}^{q_{0}}\bbR^{c^{(0)}_{m}},
		\end{equation*}
		the activation domain of
		\begin{equation*}
			\cN^{(0)}_{\bfc^{(0)},q_0}=\sigma(\bfW^{(0)}_{q_{0}}(\bigodot_{m=1}^{q_{0}-1}\sigma(\bfW^{(0)}_{m}\cdot+\bfb^{(0)}_{m}))+\cdot+\bfb^{(0)}_{q_{0}})
		\end{equation*}
		with respect to $\bfJ^{(0)}:=(J^{(0)}_{1},\cdots,J^{(0)}_{q_{0}})\in\prod_{m=1}^{q_{0}}\cD_{c^{(0)}_{m}}$ is defined recursively by
		\begin{equation*}
			D^{(0)}_{\barbfJ^{(0)}_{1}, \barbfW^{(0)}_{1}, \barbfb^{(0)}_{1}}=D_{J^{(0)}_{1},\bfW^{(0)}_{1},\bfb^{(0)}_{1}}\cap [0,1]^{d_{res}}
		\end{equation*}
		and for $2\leq m\leq q_{0} - 1$
		\begin{equation*}
			\begin{aligned}
				D^{(0)}_{\barbfJ^{(0)}_{m}, \barbfW^{(0)}_{m}, \barbfb^{(0)}_{m}}
				&=
				\left\{\bfx\in D^{(0)}_{\barbfJ^{(0)}_{m-1}, \barbfW^{(0)}_{m-1}, \barbfb^{(0)}_{m-1}}
				:(\bigodot_{m'=1}^{m-1}\sigma(\bfW^{(0)}_{m'}\cdot+\bfb^{(0)}_{m'}))(\bfx)\in D_{J^{(0)}_{m},\bfW^{(0)}_{m},\bfb^{(0)}_{m}}\right\},\\
				D^{(0)}_{\barbfJ^{(0)}_{q_{0}}, \barbfW^{(0)}_{q_{0}}, \barbfb^{(0)}_{q_{0}}}
				&=
				\left\{\bfx\in D^{(0)}_{\barbfJ^{(0)}_{q_{0}-1}, \barbfW^{(0)}_{q_{0}-1}, \barbfb^{(0)}_{q_{0}-1}}:(\bfW^{(0)}_{q_{0}}(\bigodot_{m'=1}^{q_{0}-1}\sigma(\bfW^{(0)}_{m'}\cdot+\bfb^{(0)}_{m'}))+\cdot)(\bfx)\in D_{J^{(0)}_{q_{0}},\bfI_{res},\bfb^{(0)}_{q_{0}}}\right\}.
			\end{aligned}
		\end{equation*}
		Here,
		\begin{equation*}
				\barbfW^{(0)}_{m}:=(\bfW^{(0)}_{1}, \cdots, \bfW^{(0)}_{m}),\
				\barbfb^{(0)}_{m}:=(\bfb^{(0)}_{1}, \cdots, \bfb^{(0)}_{m}),\
				\barbfJ^{(0)}_{m}:=(J^{(0)}_{1}, \cdots, J^{(0)}_{m}),
			\qquad 1\leq m\leq q_{0}.
		\end{equation*}
		For convenience, we denote $D^{(0)}_{\bfJ^{(0)}, \bfW^{(0)}, \bfb^{(0)}}:=D^{(0)}_{\barbfJ^{(0)}_{q_{0}}, \barbfW^{(0)}_{q_{0}}, \barbfb^{(0)}_{q_{0}}}$.
	\end{definition}

	With the above definitions, we could now introduce the definition of the activation domain of a multi-residual-block network. This definition is a nontrivial extension of that in \cite{ConvergenceDeepReluNet}.

	\begin{definition} {\bf (Activation domains of a multi-residual-block network)}
		\label{def:Activation domains of multi-residual-block network}
		For
		\begin{equation*}
			\barbfW_{n}:=(\bfW^{(0)},\cdots,\bfW^{(n)})\in\prod_{k=0}^{n}\prod_{m=1}^{q_{k}}\bbR^{c^{(k)}_{m}\times c^{(k)}_{m-1}},\quad and\quad
			\barbfb_{n}:=(\bfb^{(0)},\cdots,\bfb^{(n)})\in\prod_{k=0}^{n}\prod_{m=1}^{q_{k}}\bbR^{c^{(k)}_{m}},
		\end{equation*}
		the activation domain of
		\begin{equation*}
			\cN_{\bfc, \bfq, n}=\bigodot_{k=0}^{n} \cN^{(k)}_{\bfc^{(k)}, q_{k}}
		\end{equation*}
		with respect to $\barbfJ_{n}:=(\bfJ^{(0)},\cdots,\bfJ^{(n)})\in\prod_{k=0}^{n}\prod_{m=1}^{q_{n}}\cD_{c^{(k)}_{m}}$ is defined recursively by
$$
D_{\barbfJ_{0},\barbfW_{0},\barbfb_{0}}=D^{(0)}_{\bfJ^{(0)},\bfW^{(0)},\bfb^{(0)}},\ D_{\barbfJ_{n},\barbfW_{n},\barbfb_{n}}
				=
				D^{(n)}_{\bfJ^{(n)},\bfW^{(n)},\bfb^{(n)}},
$$
where
		\begin{equation*}
			\begin{aligned}
				D^{(n)}_{\barbfJ^{(n)}_{1},\barbfW^{(n)}_{1},\barbfb^{(n)}_{1}}
				&=
				\left\{
				\bfx\in D_{\barbfJ_{n-1},\barbfW_{n-1},\barbfb_{n-1}}:
				(\bigodot_{k=0}^{n-1}\cN^{(k)}_{\bfc^{(k)},q_{k}})(\bfx)\in D_{J^{(n)}_{1},\bfW^{(n)}_{1},\bfb^{(n)}_{1}}
				\right\}
				\\
				D^{(n)}_{\barbfJ^{(n)}_{m},\barbfW^{(n)}_{m},\barbfb^{(n)}_{m}}
				&=
				\biggl\{
				\bfx\in D^{(n)}_{\barbfJ^{(n)}_{m-1},\barbfW^{(n)}_{m-1},\barbfb^{(n)}_{m-1}}:
				\\
				&\qquad
				((\bigodot_{m'=1}^{m-1}\sigma(\bfW^{(n)}_{m'}\cdot+\bfb^{(n)}_{m'}))\circ (\bigodot_{k=0}^{n-1}\cN^{(k)}_{\bfc^{(k)},q_{k}}))(\bfx)\in D_{J^{(n)}_{m},\bfW^{(n)}_{m},\bfb^{(n)}_{m}}
				\biggr\},\ 2\leq m\leq q_{n}-1		
				\\
				D^{(n)}_{\barbfJ^{(n)}_{q_{n}},\barbfW^{(n)}_{q_{n}},\barbfb^{(n)}_{q_{n}}}
				&=
				\biggl\{
				\bfx\in D^{(n)}_{\barbfJ^{(n)}_{q_{n}-1},\barbfW^{(n)}_{q_{n}-1},\barbfb^{(n)}_{q_{n}-1}}:
				\\
				&\qquad
				((\bfW^{(n)}_{q_{n}}(\bigodot_{m'=1}^{q_{n}-1}\sigma(\bfW^{(n)}_{m'}\cdot+\bfb^{(n)}_{m'}))+\cdot)\circ (\bigodot_{k=0}^{n-1}\cN^{(k)}_{\bfc^{(k)},q_{k}}))(\bfx)\in D_{J^{(n)}_{q_{n}},\bfI_{res},\bfb^{(n)}_{q_{n}}}.
				\biggr\}
			\end{aligned}
		\end{equation*}
	\end{definition}

	For each non-negative number $n$, the activation domains
	\begin{equation*}
		D_{\barbfJ_{n},\barbfW_{n},\barbfb_{n}},\quad for\quad \barbfJ_{n}:=(\bfJ^{(0)},\cdots,\bfJ^{(n)})\in\prod_{k=0}^{n}\prod_{m=1}^{q_{n}}\cD_{c^{(k)}_{m}},
	\end{equation*}
	form a partition of the unit cube $[0,1]^{d_{res}}$. By using these activation domains, we are able to write down an explicit expression of the ReLU network with shortcut connections $\cN_{\bfc, \bfq, n}$ with applications of the ReLU activation function replaced by multiplications with the activation matrices. To this end, we write
	\begin{equation*}
		\prod_{i=0}^{n}\bfW_i=\bfW_{n}\bfW_{n-1}\cdots\bfW_{0}.
	\end{equation*}
	For $n,k\in\bbN$, we also adopt the following convention that
	\begin{equation*}
		\prod_{i=k}^{n}\bfW_i=\bfW_{n}\bfW_{n-1}\cdots\bfW_{k},\ {\rm for}\ n\geq k,\ {\rm and}\ \prod_{i=k}^{n}\bfW_{i}=\bfI_{d_{res}},\ {\rm for}\ n < k.
	\end{equation*}
	\begin{theorem} \label{thm:Matrix formulation of resnets}
		It holds for each $\bfx\in D_{\barbfJ_{n}, \barbfW_{n}, \barbfb_{n}}$, $\barbfJ_{n}:=(\bfJ^{(0)},\cdots,\bfJ^{(n)})\in\prod_{k=0}^{n}\prod_{m=1}^{q_{k}}\cD_{d_{res}}$ that
		\begin{equation} \label{eq:matrix formulation of resnets}
			\begin{aligned}
				\cN_{\bfc, \bfq, n}(\bfx)
				&=
				\left(\bigodot_{k=0}^{n}\cN^{(k)}_{\bfc^{(k)}, q_{k}}\right)(\bfx)=
				\left[\prod_{k=0}^{n}\left(\prod_{m=1}^{q_{k}}J^{(k)}_{m}\bfW^{(k)}_{m}+J^{(k)}_{q_{k}}\right)\right]\bfx
				\\
				&\quad+
				\sum_{k=0}^{n}\sum_{m=1}^{q_{k}}\left[\prod_{k'=k+1}^{n}\left(\prod_{m'=1}^{q_{k'}}J^{(k')}_{m'}\bfW^{(k')}_{m'}+J^{(k')}_{q_{k'}}\right)\right]
				\left(\prod_{m'=m+1}^{q_{k}}J^{(k)}_{m'}\bfW^{(k)}_{m'}\right)J^{(k)}_{m}\bfb^{(k)}_{m}.
			\end{aligned}	
		\end{equation}
	\end{theorem}
	\begin{proof}
		We proof by induction on $n$. When $n=0$
		\begin{equation*}
				\cN_{\bfc, \bfq, 0}(\bfx)
				=\cN^{(0)}_{\bfc^{(0)}, q_{k}}(\bfx)
				=\left(\sigma\left(\bfW^{(0)}_{q_{0}}\left(\bigodot_{m=1}^{q_{0}-1}\sigma(\bfW^{(0)}_{m}\cdot+\bfb^{(0)}_{m})\right)+\cdot+\bfb^{(0)}_{q_{0}}\right)\right)(\bfx).
		\end{equation*}
		Let $\bfx\in D_{\barbfJ_{0}, \barbfW_{0}, \barbfb_{0}}$. By Theorem 3.4 in \cite{ConvergenceDeepReluNet},
		\begin{equation*}
			\left(\bigodot_{m=1}^{q_{0}-1}\sigma(\bfW^{(0)}_{m}\cdot+\bfb^{(0)}_{m})\right)(\bfx)=\prod_{m=1}^{q_{0}-1}J^{(0)}_{m}\bfW^{(0)}_{m}\bfx+\sum_{m=1}^{q_{0}-1}\left(\prod_{m'=m+1}^{q_{0}-1}J^{(0)}_{m'}\bfW^{(0)}_{m'}\right)J^{(0)}_{m}\bfb^{(0)}_{m}
		\end{equation*}
		and thus for $\bfx\in D_{\barbfJ_{0}, \barbfW_{0}, \barbfb_{0}}$,
		\begin{equation} \label{eq:matrix form of one residual block}
			\begin{aligned}
				\cN_{\bfc, \bfq, 0}(\bfx)
				&=
				J^{(0)}_{q_{0}}\left(\bfW^{(0)}_{q_{0}}\left(\prod_{m=1}^{q_{0}-1}J^{(0)}_{m}\bfW^{(0)}_{m}\bfx+\sum_{m=1}^{q_{0}-1}\left(\prod_{m'=m+1}^{q_{0}-1}J^{(0)}_{m'}\bfW^{(0)}_{m'}\right)J^{(0)}_{m}\bfb^{(0)}_{m}\right)+\bfx+\bfb^{(0)}_{q_{0}}\right)
				\\
				&=\left(\prod_{m=1}^{q_{0}}J^{(0)}_{m}\bfW^{(0)}_{m}+J^{(0)}_{q_{0}}\right)\bfx+\sum_{m=1}^{q_{0}}\left(\prod_{m'=m+1}^{q_{0}}J^{(0)}_{m'}\bfW^{(0)}_{m'}\right)J^{(0)}_{m}\bfb^{(0)}_{m}.
			\end{aligned}
		\end{equation}
		The result is hence true when $n=0$. Suppose that (\ref{eq:matrix formulation of resnets}) holds for $n-1$. Now Let $\bfx\in D_{\barbfJ_{n},\barbfW_{n},\barbfb_{n}}$. Then
		\begin{equation*}
			\begin{aligned}
				\cN_{\bfc, \bfq, n}(\bfx)
				&=\cN^{(n)}_{\bfc^{(n)},q_{n}}\left(\left(\bigodot_{k=0}^{n-1}\cN^{(k)}_{\bfc^{(k)},q_{k}}\right)(\bfx)\right)
				\\
				&=\cN^{(n)}_{\bfc^{(n)},q_{n}}
				\left(
				\left[\prod_{k=0}^{n-1}\left(\prod_{m=1}^{q_{k}}J^{(k)}_{m}\bfW^{(k)}_{m}+J^{(k)}_{q_{k}}\right)\right]\bfx
				\right.
				\\
				&
				\left.
				\qquad+
				\sum_{k=0}^{n-1}\sum_{m=1}^{q_{k}}\left[\prod_{k'=k+1}^{n-1}\left(\prod_{m'=1}^{q_{k'}}J^{(k')}_{m'}\bfW^{(k')}_{m'}+J^{(k')}_{q_{k'}}\right)\right]
				\left(\prod_{m'=m+1}^{q_{k}}J^{(k)}_{m'}\bfW^{(k)}_{m'}\right)J^{(k)}_{m}\bfb^{(k)}_{m}
				\right).
			\end{aligned}
		\end{equation*}
		By definition (\ref{def:Activation domains of multi-residual-block network}), we get by (\ref{eq:matrix form of one residual block}) and induction that  for $\bfx\in D_{\barbfJ_{n},\barbfW_{n},\barbfb_{n}}$,
		\begin{equation*}
			\begin{aligned}
				\cN_{\bfc, \bfq, n}(\bfx)
				&=\cN^{(n)}_{\bfc^{(n)},q_{n}}\left(\left(\bigodot_{k=0}^{n-1}\cN^{(k)}_{\bfc^{(k)},q_{k}}\right)(\bfx)\right)
				\\
				&=\left(\prod_{m=1}^{q_{n}}J^{(n)}_{m}\bfW^{(n)}_{m}+J^{(n)}_{q_{n}}\right)\left(\left(\bigodot_{k=0}^{n-1}\cN^{(k)}_{\bfc^{(k)},q_{k}}\right)(\bfx)\right)+\sum_{m=1}^{q_{n}}\left(\prod_{m'=m+1}^{q_{n}}J^{(n)}_{m'}\bfW^{(n)}_{m'}\right)J^{(n)}_{m}\bfb^{(n)}_{m}
				\\
				&=\left(\prod_{m=1}^{q_{n}}J^{(n)}_{m}\bfW^{(n)}_{m}+J^{(n)}_{q_{n}}\right)
				\left(
				\left[\prod_{k=0}^{n-1}\left(\prod_{m=1}^{q_{k}}J^{(k)}_{m}\bfW^{(k)}_{m}+J^{(k)}_{q_{k}}\right)\right]\bfx
				\right.\\
				&\left.\qquad+
				\sum_{k=0}^{n-1}\sum_{m=1}^{q_{k}}\left[\prod_{k'=k+1}^{n-1}\left(\prod_{m'=1}^{q_{k'}}J^{(k')}_{m'}\bfW^{(k')}_{m'}+J^{(k')}_{q_{k'}}\right)\right]
				\left(\prod_{m'=m+1}^{q_{k}}J^{(k)}_{m'}\bfW^{(k)}_{m'}\right) J^{(k)}_{m}\bfb^{(k)}_{m}
				\right)
				\\
				&\qquad+
				\sum_{m=1}^{q_{n}}\left(\prod_{m'=m+1}^{q_{n}}J^{(n)}_{m'}\bfW^{(n)}_{m'}\right)J^{(n)}_{m}\bfb^{(n)}_{m}
				\\
				&=
				\left[\prod_{k=0}^{n}\left(\prod_{m=1}^{q_{k}}J^{(k)}_{m}\bfW^{(k)}_{m}+J^{(k)}_{q_{k}}\right)\right]\bfx+
				\sum_{m=1}^{q_{n}}\left(\prod_{m'=m+1}^{q_{n}}J^{(n)}_{m'}\bfW^{(n)}_{m'}\right)J^{(n)}_{m}\bfb^{(n)}_{m}
				\\
				&\qquad+
				\sum_{k=0}^{n-1}\sum_{m=1}^{q_{k}}\left[\prod_{k'=k+1}^{n}\left(\prod_{m'=1}^{q_{k'}}J^{(k')}_{m'}\bfW^{(k')}_{m'}+J^{(k')}_{q_{k}}\right)\right]
				\left(\prod_{m'=m+1}^{q_{k}}J^{(k)}_{m'}\bfW^{(k)}_{m'}\right)J^{(k)}_{m}\bfb^{(k)}_{m}
				\\
				&=
				\left[\prod_{k=0}^{n}\left(\prod_{m=1}^{q_{k}}J^{(k)}_{m}\bfW^{(k)}_{m}+J^{(k)}_{q_{k}}\right)\right]\bfx
				\\
				&\qquad+
				\sum_{k=0}^{n}\sum_{m=1}^{q_{k}}\left[\prod_{k'=k+1}^{n}\left(\prod_{m'=1}^{q_{k'}}J^{(k')}_{m'}\bfW^{(k')}_{m'}+J^{(k')}_{q_{k'}}\right)\right]
				\left(\prod_{m'=m+1}^{q_{k}}J^{(k)}_{m'}\bfW^{(k)}_{m'}\right)J^{(k)}_{m}\bfb^{(k)}_{m}
			\end{aligned}
		\end{equation*}
		which proves (\ref{thm:Matrix formulation of resnets}).
	\end{proof}

	For convenience, we denote
	\begin{equation} \label{eq:An}
		\bfA_{n}:=\prod_{k=0}^{n}\left(\prod_{m=1}^{q_{k}}J^{(k)}_{m}\bfW^{(k)}_{m}+J^{(k)}_{q_{k}}\right)
	\end{equation}
	and
	\begin{equation} \label{eq:Bn}
		\bfB_{n}:=\sum_{k=0}^{n}\sum_{m=1}^{q_{k}}\left[\prod_{k'=k+1}^{n}\left(\prod_{m'=1}^{q_{k'}}J^{(k')}_{m'}\bfW^{(k')}_{m'}+J^{(k')}_{q_{k'}}\right)\right]
		\left(\prod_{m'=m+1}^{q_{k}}J^{(k)}_{m'}\bfW^{(k)}_{m'}\right)J^{(k)}_{m}\bfb^{(k)}_{m}.
	\end{equation}
	
	
	We reach the main result of this section, which is a direct consequence of the above theorem.
	\begin{theorem} \label{thm:convergence of resnet in matrix}
		Let $\bfq:=\{q_{n}\}_{n=1}^{\infty}\subseteq \bbN_{+}$, $\bfc:=\{\bfc^{(n)}\}_{n=0}^{\infty}$ with $\bfc^{(n)}:=(c^{(n)}_{0},\cdots,c^{(n)}_{q_{n}})\in\bbN_{+}^{q_{n}}$, $\bfW:=\{\bfW^{(n)}\}_{n=0}^{\infty}$ with $\bfW^{(n)}\in\prod_{m=1}^{q_{n}}\bbR^{c^{(n)}_{m}\times c^{(n)}_{m-1}}$ be the weight matrices, and $\bfb^{(n)}:=\{\bfb^{(n)}\}_{n=0}^{\infty}$ with $\bfb^{(n)}\in\prod_{m=1}^{q_{n}}\bbR^{c^{(n)}_{m}}$ be the bias vectors.
		If for all $\bfJ=\{\bfJ^{(n)}\}_{n=1}^{\infty}$ with $\bfJ^{(n)}=(J^{(n)}_{1},\cdots,J^{(n)}_{q_{n}})$, the two limits
			\begin{equation}\label{eq:limit1}
				\lim_{n\rightarrow \infty}\bfA_{n}=\lim_{n\rightarrow \infty}\prod_{k=0}^{n}\left(\prod_{m=1}^{q_{k}}J^{(k)}_{m}\bfW^{(k)}_{m}+J^{(k)}_{q_{k}}\right)
			\end{equation}
			and
			\begin{equation}\label{eq:limit2}
				\lim_{n\rightarrow \infty}\bfB_{n}=\lim_{n\rightarrow \infty}\sum_{k=0}^{n}\sum_{m=1}^{q_{k}}\left[\prod_{k'=k+1}^{n}\left(\prod_{m'=1}^{q_{k'}}J^{(k')}_{m'}\bfW^{(k')}_{m'}+J^{(k')}_{q_{k'}}\right)\right]
				\left(\prod_{m'=m+1}^{q_{k}}J^{(k)}_{m'}\bfW^{(k)}_{m'}\right)J^{(k)}_{m}\bfb^{(k)}_{m}
			\end{equation}
both exist, then the sequence of neural networks $\{\cN_{\bfc,\bfq, n}\}_{n=1}^\infty$ converges pointwise on $[0,1]^{d_{res}}$.
	\end{theorem}

	Theorem \ref{thm:convergence of resnet in matrix} provides a basis to study the convergence of deep ReLU neural networks with shortcut connections. It reduces the problem to the existence of two limits involving infinite product of non-square matrices. Note that these two limits are much more complicated than those in \cite{ConvergenceDeepReluNet,ConvergenceDeepCNN} in that they contain double products of matrices. Also the matrices are non-square. We shall spend the next section in studying the infinite product of non-square matrices in the two limits.

\section{Convergence of DNNs with Shortcut connections} \label{sec:convergence_of_dnn_with_sc}
\setcounter{equation}{0}

	By Theorem \ref{thm:convergence of resnet in matrix}, existence of two limits (\ref{eq:limit1}) and (\ref{eq:limit2}) serves as a sufficient condition to ensure pointwise convergence of deep ReLU neural networks with shortcut connections. In particular, convergence of the infinite product of matrices
	\begin{equation} \label{eq:matrix lim 1}
		\lim_{n\rightarrow \infty}\prod_{k=0}^{n}\left(\prod_{m=1}^{q_{k}}J^{(k)}_{m}\bfW^{(k)}_{m}+J^{(k)}_{q_{k}}\right),\ \mbox{for\ any}\ J^{(n)}_{m}\in\cD_{res},
	\end{equation}
	appears in both of the limits. We hence first study this important issue in this section.
	
	Let $\|\cdot\|$ be a norm on $\mathbb{R}^{m}$ satisfying
	\begin{equation}\label{eq:nondecreasingvectornorm}
		\|\bfa\|\le\|\bfb\| \mbox{ whenever }|a_i|\le |b_i|,\ 1\le i\le m,\ \ \mbox{for}\ \bfa=(a_1,a_2,\dots,a_m), \bfb=(b_1,b_2,\dots,b_m)\in\bbR^m.
	\end{equation}
	We then define its induced matrix norm on $\bbR^{n\times m}$, also denoted by $\|\cdot\|$, by
	$$
	\|A\|=\sup_{x\in\bbR^m,x\ne0}\frac{\|Ax\|}{\|x\|},\ \ \mbox{for}\ \ A\in\bbR^{n\times m}.
	$$
	Clearly, it holds
	\begin{equation}\label{eq:matrixcon1}
		\|AB\|\le \|A\|\|B\| \mbox{ for all matrices }A,B
	\end{equation}
	and
	\begin{equation}\label{eq:matrixcon2}
		\|J_m^{(k)}\|\le 1 \mbox{ for each }J_m^{(k)}\in \cD_{d_{res}}.
	\end{equation}
	Norms satisfying the above properties include the $\|\cdot\|_p$ norms, $1\le p\le +\infty$.

	\begin{theorem} \label{thm:limit1}
		Let $\bfq=(q_{n})_{n=0}^{\infty}$ with $\|\bfq\|_{\infty}<+\infty$ and $q_{n}\in\bbN_{+}$ for $n\in\bbN$, $\bfc:=\{\bfc^{(n)}\}_{n=0}^{\infty}$ with $\bfc^{(n)}:=(c^{(n)}_{0},\cdots,c^{(n)}_{q_{n}})\in\bbN_{+}^{q_{n}}$ and $c^{(n)}_{0}=c^{(n)}_{q_{n}}=d_{res}$ for $n\in\bbN$, and $\bfW:=\{\bfW^{(n)}\}_{n=0}^{\infty}$ with  $\bfW^{(n)}:=(\bfW^{(n)}_{1},\cdots,\bfW^{(n)}_{q_{n}})\in\prod_{m=1}^{q_{n}}\bbR^{c^{(n)}_{m}\times c^{(n)}_{m-1}}$. If
		\begin{equation}\label{sufficientlimit1}
			\sum_{n=0}^{\infty}\prod_{m=1}^{q_{n}}\|\bfW^{(n)}_{m}\|<+\infty,	
		\end{equation}
		then the infinite product (\ref{eq:matrix lim 1}) converges for all $J^{(n)}_{m}\in \cD_{d_{res}}$, $m=1,\cdots,q_{n}$, $n\in\bbN$.
	\end{theorem}
	\begin{proof}
We first write
	\begin{equation*}
 \prod_{k=0}^{n}\left(\prod_{m=1}^{q_{k}}J^{(k)}_{m}\bfW^{(k)}_{m}+J^{(k)}_{q_{k}}\right)=\prod_{k=0}^{n}J^{(k)}_{q_{k}}\left(\bfW^{(k)}_{q_{k}}\prod_{m=1}^{q_{k}-1}J^{(k)}_{m}\bfW^{(k)}_{m}+\bfI_{d_{res}}\right).
	\end{equation*}
	By Theorem 4.3 in \cite{ConvergenceDeepReluNet}, a sufficient condition for convergence of limit (\ref{eq:matrix lim 1}) is
$$
\sum_{k=0}^{\infty}\biggl\|\bfW^{(k)}_{q_{k}}\prod_{m=1}^{q_{k}-1}J^{(k)}_{m}\bfW^{(k)}_{m}\biggr\|<+\infty.
 $$
 Now assume condition (\ref{sufficientlimit1}). As $\|\bfq\|_{\infty}<+\infty$, we have by (\ref{eq:matrixcon1}) and (\ref{eq:matrixcon2}) that
		\begin{equation*}
			\sum_{n=0}^{\infty}\biggl\|\bfW^{(n)}_{q_{n}}\prod_{m=1}^{q_{n}-1}J^{(n)}_{m}\bfW^{(n)}_{m}\biggr\|\leq \sum_{n=0}^{\infty}\prod_{m=1}^{q_{n}}\|\bfW^{(n)}_{m}\|<+\infty,
		\end{equation*}
		which completes the proof.
	\end{proof}
	
	We next deal with the second limit (\ref{eq:limit2}).
	\begin{theorem} \label{thm:limit2}
		Let $\bfq=(q_{n})_{n=0}^{\infty}$ with $\|\bfq\|_{\infty}<+\infty$ and $q_{n}\in\bbN_{+}$ for $n\in\bbN$, $\bfc:=\{\bfc^{(n)}\}_{n=0}^{\infty}$ with $\bfc^{(n)}:=(c^{(n)}_{0},\cdots,c^{(n)}_{q_{n}})\in\bbN_{+}^{q_{n}}$ and $c^{(n)}_{0}=c^{(n)}_{q_{n}}=d_{res}$ for $n\in\bbN$, $\bfW:=\{\bfW^{(n)}\}_{n=0}^{\infty}$ with  $\bfW^{(n)}:=(\bfW^{(n)}_{1},\cdots,\bfW^{(n)}_{q_{n}})\in\prod_{m=1}^{q_{n}}\bbR^{c^{(n)}_{m}\times c^{(n)}_{m-1}}$, and $\bfb:=\{\bfb^{(n)}\}_{n=0}^{\infty}$ with $\bfb^{(n)}:=(\bfb^{(n)}_{1},\cdots,\bfb^{(n)}_{q_{n}})\in\prod_{m=1}^{q_{n}}\bbR^{c^{(n)}_{m}}$. If
		\begin{equation} \label{eq:b lim cond1}
			\sum_{k=0}^{\infty}\sum_{m=1}^{q_{k}}\left(\prod_{m'=m+1}^{q_{k}}\|\bfW_{m'}^{(k)}\|\right)\|\bfb^{(k)}_{m}\|<+\infty,
		\end{equation}
		\begin{equation} \label{eq:b lim cond2}
			\prod_{k=n}^{\infty}\left(\prod_{m=1}^{q_{k}}J^{(k)}_{m}\bfW^{(k)}_{m}+J^{(k)}_{q_{k}}\right)\ \mbox{converges\ for\ every\ } n\in\bbN,
		\end{equation}
		and there exists a positive constant $C_1$ such that
		\begin{equation} \label{eq:b lim cond3}
			\prod_{k=n'}^{n}\left(\prod_{m=1}^{q_{k}}\|\bfW^{(k)}_{m}\|+1\right)\leq C_{1}\ \mbox{for\ all}\ n,n'\in\bbN,\ n'\leq n
		\end{equation}
		then the limit (\ref{eq:limit2}) exists.
	\end{theorem}
	\begin{proof}
		It suffices to show that
		\begin{equation*}
			\bfB_{n}=\sum_{k=0}^{n}\sum_{m=1}^{q_{k}}\left[\prod_{k'=k+1}^{n}\left(\prod_{m'=1}^{q_{k'}}J^{(k')}_{m'}\bfW^{(k')}_{m'}+J^{(k')}_{q_{k'}}\right)\right]
			\left(\prod_{m'=m+1}^{q_{k}}J^{(k)}_{m'}\bfW^{(k)}_{m'}\right)J^{(k)}_{m}\bfb^{(k)}_{m}
		\end{equation*}
		forms a Cauchy sequence in $\bbR^{d_{res}}$. By condition (\ref{eq:b lim cond1}), we could assume that there exists a positive constant $C_2$ such that
		\begin{equation} \label{eq:b lim cond4}
			\sum_{k=0}^{n}\sum_{m=1}^{q_{k}}\left(\prod_{m'=m+1}^{q_{k}}\|\bfW_{m'}^{(k)}\|\right)\|\bfb^{(k)}_{m}\|\leq C_2\ \mbox{for\ every}\ n\in\bbN.
		\end{equation}
		We define $\bfd_{n,n',n''}$ and $\bfe_{n,n''}$ for $n''< n$, $n''< n'$ and $n,n',n''\in\bbN$ as follows:
		\begin{equation} \label{eq:d n n' n''}
			\begin{aligned}
				\bfd_{n,n',n''}
				&:=
				\sum_{k=0}^{n''}\sum_{m=1}^{q_{k}}
				\left[
				\prod_{k'=k+1}^{n}\left(\prod_{m'=1}^{q_{k'}}J^{(k')}_{m'}\bfW^{(k')}_{m'}+J^{(k')}_{q_{k'}}\right)
				\right.
				\\	
				&\left.\bqquad
				-\prod_{k'=k+1}^{n'}\left(\prod_{m'=1}^{q_{k'}}J^{(k')}_{m'}\bfW^{(k')}_{m'}+J^{(k')}_{q_{k'}}\right)
				\right]
				\left(\prod_{m'=m+1}^{q_{k}}J^{(k)}_{m'}\bfW^{(k)}_{m'}\right)J^{(k)}_{m}\bfb^{(k)}_{m}
			\end{aligned}
		\end{equation}
		and
		\begin{equation} \label{eq:e n n'}
			\bfe_{n,n''}:=\sum_{k=n''+1}^{n}\sum_{m=1}^{q_{k}}\left[\prod_{k'=k+1}^{n}\left(\prod_{m'=1}^{q_{k'}}J^{(k')}_{m'}\bfW^{(k')}_{m'}+J^{(k')}_{q_{k'}}\right)\right]\left(\prod_{m'=m+1}^{q_{k}}J^{(k)}_{m'}\bfW^{(k)}_{m'}\right)J^{(k)}_{m}\bfb^{(k)}_{m}.
		\end{equation}
		Let $\epsilon>0$ be arbitrary. By condition (\ref{eq:b lim cond1}), there exists a positive integer $n''$ such that
		\begin{equation} \label{eq:b epsilon 1}
			\sum_{k=n''+1}^{n}\sum_{m=1}^{q_{k}}\left(\prod_{m'=m+1}^{q_{k}}\|\bfW_{m'}^{(k)}\|\right)\|\bfb^{(k)}_{m}\|<\frac{\epsilon}{3C_{1}},\ \forall n > n''.
		\end{equation}
		Thus, by the triangle inequality, (\ref{eq:matrixcon1}), (\ref{eq:matrixcon2}), (\ref{eq:b lim cond3}), (\ref{eq:e n n'}) and (\ref{eq:b epsilon 1}), we have
		\begin{equation} \label{eq:b epsilon 2}
			\begin{aligned}
				\|\bfe_{n,n''}\|
				&=\left\|\sum_{k=n''+1}^{n}\sum_{m=1}^{q_{k}}\left[\prod_{k'=k+1}^{n}\left(\prod_{m'=1}^{q_{k'}}J^{(k')}_{m'}\bfW^{(k')}_{m'}+J^{(k')}_{q_{k'}}\right)\right]\left(\prod_{m'=m+1}^{q_{k}}J^{(k)}_{m'}\bfW^{(k)}_{m'}\right)J^{(k)}_{m}\bfb^{(k)}_{m}\right\|
				\\
				&\leq
				\sum_{k=n''+1}^{n}\sum_{m=1}^{q_{k}}\left[\prod_{k'=k+1}^{n}\left(\prod_{m'=1}^{q_{k'}}\|\bfW^{(k')}_{m'}\|+1\right)\right]\left(\prod_{m'=m+1}^{q_{k}}\|\bfW^{(k)}_{m'}\|\right)\|\bfb^{(k)}_{m}\|
				\\
				&\leq
				C_{1}\sum_{k=n''+1}^{n}\sum_{m=1}^{q_{k}}\left(\prod_{m'=m+1}^{q_{k}}\|\bfW^{(k)}_{m'}\|\right)\|\bfb^{(k)}_{m}\|
				\\
				&\leq \frac{\epsilon}{3}.
			\end{aligned}
		\end{equation}
		By (\ref{eq:b lim cond2}), for big enough $n,n'>n''$, it holds that
		\begin{equation} \label{eq:b epsilon 3}
			\left\|\prod_{k'=k+1}^{n}\left(\prod_{m'=1}^{q_{k'}}J^{(k')}_{m'}\bfW^{(k')}_{m'}+J^{(k')}_{q_{k'}}\right)-\prod_{k'=k+1}^{n'}\left(\prod_{m'=1}^{q_{k'}}J^{(k')}_{m'}\bfW^{(k')}_{m'}+J^{(k')}_{q_{k'}}\right)\right\|<\frac{\epsilon}{3C_{2}},\ \forall k\leq n'.
		\end{equation}
		Thus, by the triangle inequality, (\ref{eq:matrixcon1}), (\ref{eq:matrixcon2}), (\ref{eq:b lim cond4}), (\ref{eq:d n n' n''}) and (\ref{eq:b epsilon 3}), we have
		\begin{equation} \label{eq:b epsilon 4}
			\begin{aligned}
				\|\bfd_{n,n',n''}\|
				&=
				\left\|
				\sum_{k=0}^{n''}\sum_{m=1}^{q_{k}}
				\left[
				\prod_{k'=k+1}^{n}\left(\prod_{m'=1}^{q_{k'}}J^{(k')}_{m'}\bfW^{(k')}_{m'}+J^{(k')}_{q_{k'}}\right)
				\right.
				\right.
				\\	
				&\left.\left.\bqquad
				-\prod_{k'=k+1}^{n'}\left(\prod_{m'=1}^{q_{k'}}J^{(k')}_{m'}\bfW^{(k')}_{m'}+J^{(k')}_{q_{k'}}\right)
				\right]
				\left(\prod_{m'=m+1}^{q_{k}}J^{(k)}_{m'}\bfW^{(k)}_{m'}\right)J^{(k)}_{m}\bfb^{(k)}_{m}
				\right\|
				\\
				&\leq
				\sum_{k=0}^{n''}\sum_{m=1}^{q_{k}}
				\left\|
				\prod_{k'=k+1}^{n}\left(\prod_{m'=1}^{q_{k'}}J^{(k')}_{m'}\bfW^{(k')}_{m'}+J^{(k')}_{q_{k'}}\right)
				\right.
				\\	
				&\left.\bqquad
				-\prod_{k'=k+1}^{n'}\left(\prod_{m'=1}^{q_{k'}}J^{(k')}_{m'}\bfW^{(k')}_{m'}+J^{(k')}_{q_{k'}}\right)
				\right\|
				\left(\prod_{m'=m+1}^{q_{k}}\|\bfW^{(k)}_{m'}\|\right)\|\bfb^{(k)}_{m}\|
				\\
				&\leq
				\sum_{k=0}^{n''}\sum_{m=1}^{q_{k}}\frac{\epsilon}{3C_{2}}\left(\prod_{m'=m+1}^{q_{k}}\|\bfW_{m'}^{(k)}\|\right)\|\bfb^{(k)}_{m}\|
				\\
				&\leq
				\frac{\epsilon}{3C_{2}}\sum_{k=0}^{n''}\sum_{m=1}^{q_{k}}\left(\prod_{m'=m+1}^{q_{k}}\|\bfW_{m'}^{(k)}\|\right)\|\bfb^{(k)}_{m}\|
				\\
				&\leq \frac{\epsilon}{3}.
			\end{aligned}
		\end{equation}
		Since $\|\bfB_{n}-\bfB_{n'}\|=\|\bfe_{n,n''}+\bfd_{n,n',n''}-\bfe_{n',n''}\|$, by the triangle inequality, (\ref{eq:b epsilon 2}) and (\ref{eq:b epsilon 4}) we have
		\begin{equation}
				\|\bfB_{n}-\bfB_{n'}\|
				=
				\|\bfe_{n,n''}+\bfd_{n,n',n''}-\bfe_{n',n''}\|
				\leq
				\|\bfe_{n,n''}\|+\|\bfd_{n,n',n''}\|+\|\bfe_{n',n''}\|
				<\epsilon.
		\end{equation}
		This shows $\bfB_{n}$ is a Cauchy sequence and thus it converges.
	\end{proof}

We now apply Theorem \ref{thm:limit1} and Theorem \ref{thm:limit2} to establish the convergence of DNNs $\cN_{\bfc, \bfq, n}$ with Shortcut connections.
	\begin{theorem} \label{thm:convergence of cn}
		Let $\bfq=(q_{n})_{n=0}^{\infty}$ with $\|\bfq\|_{\infty}<+\infty$ and $q_{n}\in\bbN_{+}$ for $n\in\bbN$, $\bfc:=\{\bfc^{(n)}\}_{n=0}^{\infty}$ with $\bfc^{(n)}:=(c^{(n)}_{0},\cdots,c^{(n)}_{q_{n}})\in\bbN_{+}^{q_{n}}$ and $c^{(n)}_{0}=c^{(n)}_{q_{n}}=d_{res}$ for $n\in\bbN$, $\bfW:=\{\bfW^{(n)}\}_{n=0}^{\infty}$ with  $\bfW^{(n)}:=(\bfW^{(n)}_{1},\cdots,\bfW^{(n)}_{q_{n}})\in\prod_{m=1}^{q_{n}}\bbR^{c^{(n)}_{m}\times c^{(n)}_{m-1}}$, and $\bfb:=\{\bfb^{(n)}\}_{n=0}^{\infty}$ with $\bfb^{(n)}:=(\bfb^{(n)}_{1},\cdots,\bfb^{(n)}_{q_{n}})\in\prod_{m=1}^{q_{n}}\bbR^{c^{(n)}_{m}}$. If
		\begin{equation} \label{eq:cN lim cond1}
			\sum_{k=0}^{\infty}\prod_{m=1}^{q_{k}}\|\bfW^{(k)}_{m}\|<+\infty
		\end{equation}
		and
		\begin{equation} \label{eq:cN lim cond2}
			\sum_{k=0}^{\infty}\sum_{m=1}^{q_{k}}\left(\prod_{m'=m+1}^{q_{k}}\|\bfW_{m'}^{(k)}\|\right)\|\bfb^{(k)}_{m}\|<+\infty
		\end{equation}
		then $\cN_{\bfc, \bfq, n}$ converges pointwise on $[0,1]^{d_{res}}$.
	\end{theorem}
	\begin{proof}
		By (\ref{eq:cN lim cond1}), we could assume $$\sum_{k=0}^{\infty}\prod_{m=1}^{q_{k}}\|\bfW^{(k)}_{m}\|=C_{1}<+\infty.$$
		Furthermore, we can verify that for all $n,n'\in\bbN,\ n'\leq n$,
		\begin{equation*}
				\prod_{k=n'}^{n}\left(\prod_{m=1}^{q_{k}}\|\bfW^{(k)}_{m}\|+1\right)
				\leq
				\prod_{k=n'}^{n}\exp\left(\prod_{m=1}^{q_{k}}\|\bfW^{(k)}_{m}\|\right)
				\leq
				\exp\left(\sum_{k=n'}^{n}\prod_{m=1}^{q_{k}}\|\bfW^{(k)}_{m}\|\right)
				\leq \exp(C_{1}).
		\end{equation*}
		Thus, by Theorem \ref{thm:limit1} and \ref{thm:limit2}, limits (\ref{eq:limit1}) and (\ref{eq:limit2}) exists for all $J^{(n)}_{m}\in\cD_{d_{res}}$. Therefore, $\cN_{\bfc, \bfq, n}$ converges pointwise on $[0,1]^{d_{res}}$.
	\end{proof}
	
	Since $\bfW^{(0)}=(\bfW^{(0)}_{1})$ with $\bfW^{(0)}_{1}=\begin{bmatrix}
		\bfW_{s} & {\bf 0}
	\end{bmatrix}$, $\bfb^{(0)}=(\bfb^{(0)}_{1})$ with $\bfb^{(0)}_{1}=\bfb_{s}$, and the network (\ref{eq:structure of DNN with sc}) without the linear output part is equivalent to $\left(\cN_{\bfc, \bfq, n}\circ \cI \right)(\bfx)$ for $\bfx\in[0,1]^{d_{in}}$ where $\cI(\bfx):=\begin{bmatrix}\bfI_{in} \\ {\bf 0}\end{bmatrix}\bfx$ and $\begin{bmatrix}\bfI_{in} \\ {\bf 0}\end{bmatrix}\in\bbR^{d_{res}\times d_{in}}$, we reach the main theorem of the section.
	\begin{theorem} \label{thm:convergence of DNN with sc}
		Let $\bfq=(q_{n})_{n=0}^{\infty}$ with $q_{0}=1$, $q_{n}\in\bbN_{+}$ for $n\in\bbN$, and $\|\bfq\|_{\infty}<+\infty$, $\bfc:=\{\bfc^{(n)}\}_{n=0}^{\infty}$ with $\bfc^{(n)}:=(c^{(n)}_{0},\cdots,c^{(n)}_{q_{n}})\in\bbN_{+}^{q_{n}}$ and $c^{(n)}_{0}=c^{(n)}_{q_{n}}=d_{res}$ for $n\in\bbN$, $\bfW:=\{\bfW^{(n)}\}_{n=0}^{\infty}$ with  $\bfW^{(n)}:=(\bfW^{(n)}_{1},\cdots,\bfW^{(n)}_{q_{n}})\in\prod_{m=1}^{q_{n}}\bbR^{c^{(n)}_{m}\times c^{(n)}_{m-1}}$ and $\bfW^{(0)}=(\bfW^{(0)}_{1})$ with $\bfW^{(0)}_{1}=\begin{bmatrix}
			\bfW_{s} & {\bf 0}
		\end{bmatrix}$, and $\bfb:=\{\bfb^{(n)}\}_{n=0}^{\infty}$ with $\bfb^{(n)}:=(\bfb^{(n)}_{1},\cdots,\bfb^{(n)}_{q_{n}})\in\prod_{m=1}^{q_{n}}\bbR^{c^{(n)}_{m}}$. If
		\begin{equation} \label{eq:DNN with sc cond1}
			\sum_{k=0}^{\infty}\prod_{m=1}^{q_{k}}\|\bfW^{(k)}_{m}\|<+\infty
		\end{equation}
		and
		\begin{equation} \label{eq:DNN with sc cond2}
			\sum_{k=0}^{\infty}\left(\prod_{m'=m+1}^{q_{k}}\|\bfW_{m'}^{(k)}\|\right)\sum_{m=1}^{q_{k}}\|\bfb^{(k)}_{m}\|<+\infty
		\end{equation}
		then network (\ref{eq:structure of DNN with sc}) converges pointwise on $[0,1]^{d_{in}}$.
	\end{theorem}

	\section{Convergence of Deep Residual Networks} \label{sec:convergence_of_drn}
\setcounter{equation}{0}

	We now return to the main topic of this paper, which is to establish the convergence of deep residual networks. We will apply the result on convergence of deep neural network with shortcut connections in Section \ref{sec:convergence_of_dnn_with_sc} to deep ResNets.
	
	We shall work with the matrix norm induced by the $\ell^p$ vector norm. Recall the Riesz-Thorin interpolation theorem (see, \cite{Folland}, page 200) that for any matrix $A$ and $p\in[1,+\infty]$
	\begin{equation} \label{eq:norm ine}
		\|A\|_{p} \leq \|A\|_{1}^{\frac{1}{p}}\|A\|_{\infty}^{1-\frac{1}{p}}.
	\end{equation}
Let us first use this interpolation theorem to study the relationship between the norm of filter masks and the norm of the matrices in (\ref{eq:T(w) 2-d multiple channels}) associated with the filter masks.
	\begin{lemma} \label{lemma:norm of T(w) and w}
		Let $\bfw:=(\bfw_{1},\cdots,\bfw_{c_{out}})$ with $\bfw_{i}:=(\bfw_{i,1},\cdots,\bfw_{i,c_{in}})\in\bbR^{(2f+1)\times(2f+1)\times c_{in}}$ and $\bfw_{i,j}:=(w_{i,j,i',j'})_{i'=1,j'=1}^{2f+1, 2f+1}$ for $1\leq i\leq c_{out}$ denote the filter masks for multi-channel convolution, $\bfW:=\bfT(\bfw)$ be defined as in (\ref{eq:T(w) 2-d multiple channels}), and $\|\cdot\|_{p}$ be $\ell_{p}$-norm on matrices. It holds that
		\begin{equation} \label{eq:norm of T(w) and w}
			\|\bfW\|_{p}\leq \left(\max_{1\leq j\leq c_{in}}\sum_{i=1}^{c_{out}}\sum_{i'=1}^{2f+1}\sum_{j'=1}^{2f+1}|w_{i,j,i',j'}|\right)^{\frac{1}{p}}\left(\max_{1\leq i\leq c_{out}}\sum_{j=1}^{c_{in}}\sum_{i'=1}^{2f+1}\sum_{j'=1}^{2f+1}|w_{i,j,i',j'}|\right)^{1-\frac{1}{p}}
		\end{equation}
	\end{lemma}
	\begin{proof}
		By the Riesz-Thorin interpolation theorem (\ref{eq:norm ine}), we have
		\begin{equation} \label{eq:T(w) ine 1}
			\|\bfW\|_{p}\leq\|\bfW\|_{1}^{\frac{1}{p}}\|\bfW\|_{\infty}^{1-\frac{1}{p}}.
		\end{equation}
		Since
		\begin{equation*}
			\bfW=\begin{bmatrix}
				\bfT(\bfw_{1,1}) & \bfT(\bfw_{1,2}) & \cdots & \bfT(\bfw_{1,c_{in}})\\
				\bfT(\bfw_{2,1}) & \bfT(\bfw_{2,2}) & \cdots & \bfT(\bfw_{2,c_{in}})\\
				\vdots & \vdots & \ddots & \vdots\\
				\bfT(\bfw_{c_{out},1}) & \bfT(\bfw_{c_{out},2}) & \cdots & \bfT(\bfw_{c_{out},c_{in}})\\
			\end{bmatrix},
		\end{equation*}
		we obtain
		\begin{equation} \label{eq:T(w) ine 2}
			\|\bfW\|_{1}\leq\max_{1\leq j\leq c_{in}}\sum_{i=1}^{c_{out}}\|\bfT(\bfw_{i,j})\|_{1}
		\end{equation}
		and
		\begin{equation} \label{eq:T(w) ine 3}
			\|\bfW\|_{\infty}\leq\max_{1\leq i\leq c_{out}}\sum_{j=1}^{c_{in}}\|\bfT(\bfw_{i,j})\|_{\infty}.
		\end{equation}
		Moreover, by (\ref{eq:T(w) element}) and (\ref{eq:T(w) 2-d single channel}), we have for $i=1,\cdots,c_{out}$ and $j=1,\cdots,c_{in}$,
		\begin{equation} \label{eq:T(w) ine 4}
			\|\bfT(\bfw_{i,j})\|_{1}\leq\sum_{i'=1}^{2f+1}\sum_{j'=1}^{2f+1}|w_{i,j,i',j'}|
		\end{equation}
		and
		\begin{equation} \label{eq:T(w) ine 5}
			\|\bfT(\bfw_{i,j})\|_{\infty}\leq\sum_{i'=1}^{2f+1}\sum_{j'=1}^{2f+1}|w_{i,j,i',j'}|.
		\end{equation}
		Combining (\ref{eq:T(w) ine 1}), (\ref{eq:T(w) ine 2}), (\ref{eq:T(w) ine 3}), (\ref{eq:T(w) ine 4}) and (\ref{eq:T(w) ine 5}), we get
		\begin{equation*}
			\begin{aligned}
					\|\bfW\|_{p}
					&\leq
					\|\bfW\|_{1}^{\frac{1}{p}}\|\bfW\|_{\infty}^{1-\frac{1}{p}}
					\\
					&\leq
					\left(\max_{1\leq j\leq c_{in}}\sum_{i=1}^{c_{out}}\|\bfT(\bfw_{i,j})\|_{1}\right)^{\frac{1}{p}}\left(\max_{1\leq i\leq c_{out}}\sum_{j=1}^{c_{in}}\|\bfT(\bfw_{i,j})\|_{\infty}\right)^{1-\frac{1}{p}}
					\\
					&\leq
					\left(\max_{1\leq j\leq c_{in}}\sum_{i=1}^{c_{out}}\sum_{i'=1}^{2f+1}\sum_{j'=1}^{2f+1}|w_{i,j,i',j'}|\right)^{\frac{1}{p}}\left(\max_{1\leq i\leq c_{out}}\sum_{j=1}^{c_{in}}\sum_{i'=1}^{2f+1}\sum_{j'=1}^{2f+1}|w_{i,j,i',j'}|\right)^{1-\frac{1}{p}}
			\end{aligned}
		\end{equation*}
		and prove the lemma.
	\end{proof}
	
	We could now apply Theorem \ref{thm:convergence of DNN with sc} with Lemma \ref{lemma:norm of T(w) and w} to obtain convergence of ResNets.
	\begin{theorem} \label{thm:convergence of ResNet}
		Let
\begin{itemize}
\item $\bfq=(q_{n})_{n=0}^{\infty}$ with $q_{0}=1$, $q_{n}\in\bbN_{+}$ for $n\in\bbN$, and $\|\bfq\|_{\infty}<+\infty$,
\item $\bfc:=\{\bfc^{(n)}\}_{n=0}^{\infty}$ with $\bfc^{(n)}:=(c^{(n)}_{0},\cdots,c^{(n)}_{q_{n}})\in\bbN_{+}^{q_{n}}$ and $\sup_{n\in\bbN}\|\bfc^{(n)}\|_{\infty}<+\infty$,
\item   $\bff:=(\bff^{(n)})_{n=0}^{\infty}$ with $\bff^{(n)}:=(f^{(n)}_{1},\cdots,f^{(n)}_{q_{n}})\in\bbN_{+}^{q_{n}}$, and $\sup_{n\in\bbN}\|\bff^{(n)}\|_{\infty}<+\infty$,
\item	$\bfw:=\{\bfw^{(n)}\}_{n=0}^{\infty}$ with  $\bfw^{(n)}:=(\bfw^{(n)}_{1},\cdots,\bfw^{(n)}_{q_{n}})\in\prod_{m=1}^{q_{n}}\bbR^{(2f^{(n)}_{m}+1)\times(2f^{(n)}_{m}+1)\times c^{(n)}_{m}\times c^{(n)}_{m-1}}$, and
     \item $\bfw^{(n)}_{m}:=(\bfw^{(n)}_{m,i,j})_{i=1,j=1}^{c^{(n)}_{m},c^{(n)}_{m-1}}$, $\bfw^{(n)}_{m,i,j}:=(w^{(n)}_{i,j,i',j'})_{i'=1,j'=1}^{2f^{(n)}_{m}+1,2f^{(n)}_{m}+1}$,
		$\bfw^{(0)}=(\bfw^{(0)}_{1})$ with $\bfw^{(0)}_{1}=\bfw_{s}$,
		\item $\bfb:=\{\bfb^{(n)}\}_{n=0}^{\infty}$ with $\bfb^{(n)}:=(\bfb^{(n)}_{1},\cdots,\bfb^{(n)}_{q_{n}})\in\prod_{m=1}^{q_{n}}\bbR^{d\times d\times c^{(n)}_{m}}$ and $\bfb^{(n)}_{m}:=(b^{(n)}_{m,1}\bfE_{d},\cdots, b^{(n)}_{m,c^{(n)}_{m}}\bfE_{d})$.
\end{itemize}
If
		\begin{equation} \label{eq:resnet cond 1}
			\sum_{n=0}^{\infty}\prod_{m=1}^{q_{n}}
			\left(\max_{1\leq j\leq c_{m-1}^{(n)}}\sum_{i=1}^{c_{m}^{(n)}}\sum_{i'=1}^{2f+1}\sum_{j'=1}^{2f+1}|w_{i,j,i',j'}|\right)^{\frac{1}{p}}\left(\max_{1\leq i\leq c_{m}^{(n)}}\sum_{j=1}^{c_{m-1}^{(n)}}\sum_{i'=1}^{2f+1}\sum_{j'=1}^{2f+1}|w_{i,j,i',j'}|\right)^{1-\frac{1}{p}}<+\infty
		\end{equation}
		and
		\begin{equation} \label{eq:resnet cond 2}
			\sum_{n=0}^{\infty}\sum_{m=1}^{q_{n}}P_{m}^{(n)}\left(\sum_{i=1}^{c^{(n)}_{m}}|b^{(n)}_{m}|^{p}\right)^{\frac{1}{p}}<+\infty
		\end{equation}
		where
		\begin{equation}
			P_{m}^{(n)} = \prod_{m'=m+1}^{q_{n}}
			\left[
			\left(\max_{1\leq j\leq c_{m'-1}^{(n)}}\sum_{i=1}^{c_{m'}^{(n)}}\sum_{i'=1}^{2f+1}\sum_{j'=1}^{2f+1}|w_{i,j,i',j'}|\right)^{\frac{1}{p}}
			\left(\max_{1\leq i\leq c_{m'}^{(n)}}\sum_{j=1}^{c_{m'-1}^{(n)}}\sum_{i'=1}^{2f+1}\sum_{j'=1}^{2f+1}|w_{i,j,i',j'}|\right)^{1-\frac{1}{p}}
			\right],
		\end{equation}
		then network (\ref{eq:structure of ResNet}) converges pointwise on $[0,1]^{d\times d\times c_{in}}$.
	\end{theorem}
	\begin{proof}
		  Note that network (\ref{eq:structure of ResNet}) is equivalent to (\ref{eq:structure of DNN with sc}) where the weight matrices associated with the filter masks and the bias vectors are given by $\bfW:=\{\bfW^{(n)}\}_{n=0}^{\infty}$ with $\bfW^{(n)}:=(\bfW^{(n)}_{1},\cdots,\bfW^{(n)}_{q_{n}})$, $\bfW^{(n)}_{m}=\bfT(\bfw^{(n)}_{m})\in\bbR^{d^{2}c^{(n)}_{m}\times d^{2}c^{(n)}_{m-1}}$ and $\bfW_{s}=\bfT(\bfw_{s})$, and $\tbfb:=\{\tbfb^{(n)}\}_{n=0}^{\infty}$ with $\tbfb^{(n)}:=(\tbfb^{(n)}_{1},\cdots,\tbfb^{(n)}_{q_{n}})$, $\tbfb^{(n)}_{m}=\vec(\bfb^{(n)}_{m})\in\bbR^{d^{2}c^{(n)}_{m}}$ and $\tbfb_{s}=\vec(\bfb_{s})$.
		Also, by Lemma \ref{lemma:norm of T(w) and w} and condition (\ref{eq:resnet cond 1}), (\ref{eq:DNN with sc cond1}) holds under the assumptions of this theorem. Furthermore, $\sup_{n\in\bbN}\|\bfc^{(n)}\|_{\infty}<+\infty$ and (\ref{eq:resnet cond 2}) ensure that (\ref{eq:DNN with sc cond2}) holds true as well. Thus, network (\ref{eq:structure of ResNet}) converges pointwise on $[0,1]^{d\times d\times c_{in}}$.
	\end{proof}

	From the theoretical results above, one sees that if the conditions (\ref{eq:resnet cond 1}) and (\ref{eq:resnet cond 2}) hold, the network will learn the identity maps as the depth tends to infinity, which verifies the design spirit of ResNets \cite{KaimingHe}.

\section{Numerical Experiments} \label{sec:numerical_experiments}
\setcounter{equation}{0}
	In this section, we shall conduct experiments with deep ResNets on standard image classification benchmarks CIFAR10 to verify our theory with the numerical results. To be specific, we shall train a sufficient deep ResNet on the benchmark data until an over 90\% accuracy is achieved. In this situation, the network is considered to be convergent. We then compute the partial sums in the series in the sufficient conditions (\ref{eq:resnet cond 1}) and (\ref{eq:resnet cond 2}). If the partial sums are bounded as the depth of the ResNet increases then the theoretical result Theorem \ref{thm:convergence of ResNet} on the convergence of the ResNets is verified.

We now described our experiments. Our architecture is outlined in table \ref{table:1}. Each residual block has the form
$$
\relu\left(\bfx + \bfw_{4}\bigodot_{m=1}^{3}\relu\left(\bfw_{m}\ast \cdot + \bfb_{m}\right)(\bfx) + \bfb_{4}\right),
 $$
 where $\bfw_{m}, \bfb_{m}, i=1,\cdots,4$ are filter masks and bias vectors with dimensions determined by the shape of filter masks and the number of channels. All the convolutions in our network are with stride 1, and zero padding are used in all convolutions except for the global average pooling in the end of our network. Note that our transformation is dimensionality-preserving, and the dimensions are always matched between the head and the tail of residual blocks.
	
	We implemented and trained our model with PyTorch framework, using a momentum optimizer with momentum 0.9, and batch size 128. The technique of weight decayness is not involved in the training process of our model. The initial learning rate is 0.05, which drops by a factor 10 at 60, 90 and 120 epochs. The model reaches peak performance at around 30k steps for CIFAR10, which takes about 20$h$ on a single NVIDIA GTX3090 GPU. Our code can be easily derived from an open source implementation\footnote{https://github.com/pytorch/vision/blob/main/torchvision/models/resnet.py} by adjusting the residual components and model architecture. All of the initial parameters in our code are generated by the default methods in PyTorch.
	
	The major difference between our model and the standard models in \cite{KaimingHe} is that no max pooling or batch normalization is involved in our model and thus we add biases in our model. Though the structure and techniques we used in our model are quite simple, we still obtain $91\%$ accuracy on the benchmark of CIFAR10.
	\begin{table}[htbp]
		\centering
		\caption{Architecture for CIFAR10 (241 convolutions, 6.4M parameters)}
		\label{table:1}
		\begin{tabular}{|c|c|c|}
			\hline
			weight dimensions & bias dimensions & description \\
			\hline
			3$\times$3$\times$3$\times$256 & 256 & 1 standard convolution \\
			\hline
			$
				\begin{bmatrix}
					1\times1\times256\times64 \\
					3\times3\times64\times64 \\
					3\times3\times64\times64 \\
					1\times1\times64\times256
				\end{bmatrix}
			$
			&
			$
			\begin{bmatrix}
				64 \\
				64 \\
				64 \\
				256
			\end{bmatrix}
			$
			&
			60 residual blocks\\
			\hline
			- & - & 32$\times$32 global average pooling \\
			256$\times$classes & classes & linear map \\
			\hline
		\end{tabular}
	\end{table}
	\begin{figure}[htbp]
		\centering
		\label{fig:accuracy}
		\includegraphics[scale=0.5]{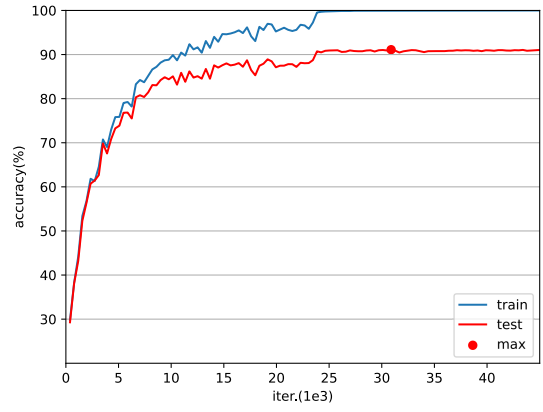}
		\caption{Convergence plot of best model for CIFAR10}
	\end{figure}
	\begin{figure}[htbp]
		\centering
		\label{fig:sufficient cond}
		\includegraphics[scale=0.5]{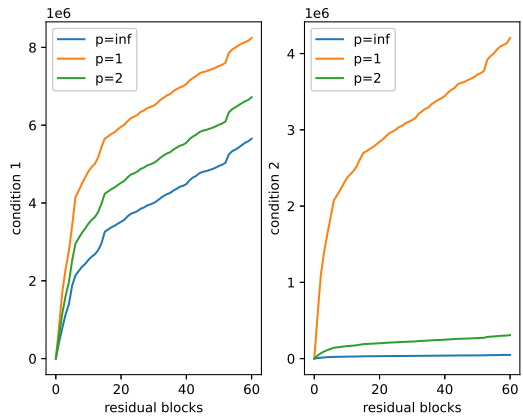}
		\caption{Numerical results for two sufficient conditions}
	\end{figure}

	As is shown by figure \ref{fig:sufficient cond}, the sufficient conditions in Theorem \ref{thm:convergence of ResNet} are indeed satisfied.

	\newpage	
	{\small
	\bibliographystyle{amsplain}

\begin{thebibliography}{30}
		
		\bibitem{Adcock} B. Adcock and N. Dexter, The gap between theory and practice in function approximation with deep neural networks, {\it SIAM J. Math. Data Sci.} \textbf{3} (2021), no. 2, 624--655.
		
		\bibitem{AlexNet} A. Krizhevsky, I. Sutskever, and G. E. Hinton, Imagenet classification with deep convolutional neural networks, {\it Adv. Neural Inf. Process. Syst.} \textbf{25} (2012).
		
		\bibitem{AllenZhu} Z. Allen-Zhu, and Y. Li, What can resnet learn efficiently, going beyond kernels?, {\it Adv. Neural Inf. Process. Syst.} \textbf{32} (2019).
		
		\bibitem{Artzrouni} M. Artzrouni, On the convergence of infinite products of matrices,  {\it Linear Algebra Appl.} \textbf{74} (1986), 11--21.
		
		
		\bibitem{GoogleLeNet} S. Christian, W. Liu, Y. Jia, P. Sermanet, S. Reed, D. Anguelov, D. Erhan, V. Vanhoucke, and A. Rabinovich, Going deeper with convolutions, 2015 IEEE Conference on Computer Vision and Pattern Recognition (CVPR), 2015, 1-9.
		
		\bibitem{I.Daubechies} I. Daubechies, {\it Ten Lectures on Wavelets}, SIAM, Philadelphia, 1992.
		
		\bibitem{Devore1} R. DeVore, B. Hanin, and G. Petrova, Neural network approximation, arXiv: 2012.14501v1, 2020.
		
		\bibitem{E} W. E and Q. Wang, Exponential convergence of the deep neural network approximation for analytic functions, {\it Sci. China Math.} \textbf{61} (2018), no. 10, 1733-1740.
		
		\bibitem{Elbrachter} D. Elbr\"{a}chter, D. Perekrestenko, P. Grohs, and H. B\"{o}lcskei, Deep neural network approximation theory, arXiv:1901.02220.
		
		\bibitem{Folland2} G. Folland, {\it Fourier analysis and its applications}, Vol. 4. American Mathematical Soc., 2009.
		
		\bibitem{Folland} G. Folland, {\it Real Analysis: Modern Techniques and Their Applications}, John Wiley \& Sons, 1999.
		
		\bibitem{Frei} S. Frei, Y. Cao, and Q. Gu, Algorithm-dependent generalization bounds for overparameterized deep residual networks, {\it Adv. Neural Inf. Process. Syst.} \textbf{32} (2019).
		
		\bibitem{Goodfellow} I. Goodfellow, Y. Bengio, and A. Courville, {\it Deep Learning}, MIT Press, Cambridge, 2016.
		
		\bibitem{Daubechies} I. Daubechies, R. DeVore, S. Foucart, B. Hanin, and G. Petrova, Nonlinear approximation and (deep) ReLU networks, arXiv: 1905.02199, 2019.
		
		\bibitem{He} F. He, T. Liu, and D. Tao, Why ResNet Works? Residuals Generalize, {\it IEEE Trans. Neural Netw. Learn. Syst.} \textbf{31}, no. 12 (2020): 5349-5362.
		
		\bibitem{KaimingHe} K. He, X. Zhang, S. Ren, and J. Sun, Deep residual learning for image recognition, 2016 IEEE Conference on Computer Vision and Pattern Recognition (CVPR), 2016, 770--778.
		
		\bibitem{KaimingHe2} K. He, X. Zhang, S. Ren, and J. Sun, Identity mappings in deep residual networks, In: B. Leibe, J. Matas, N. Sebe, M. Welling (eds) Computer Vision 篓C ECCV 2016, Lecture Notes in Computer Science, vol. \textbf{9908}, Springer, Cham.
		
		\bibitem{KaiXuan} K. Huang, Y. Wang, M. Tao, and T. Zhao, Why Do Deep Residual Networks Generalize Better than Deep Feedforward Networks?---A Neural Tangent Kernel Perspective, {\it Adv. Neural Inf. Process. Syst.} \textbf{33} (2020): 2698-2709.
		
	
		
		\bibitem{VGG} S. Karen, and A. Zisserman, Very deep convolutional networks for large-scale image recognition, arXiv preprint arXiv:1409.1556 (2014).
		
		\bibitem{Lax} P. D. Lax, {\it Functional Analysis}, Wiley-Interscience, New York, 2002.
		
		\bibitem{LeCun} Y. LeCun, Y. Bengio, and G. Hinton, Deep learning, {\it Nature} \textbf{521} (2015), no. 7553, 436-444, 2015.
		
		\bibitem{Lin} H. Lin, and S. Jegelka, Resnet with one-neuron hidden layers is a universal approximator, {\it Adv. Neural Inf. Process. Syst.} \textbf{31} (2018).
		
		\bibitem{Lu} Y. Lu, C. Ma, Y. Lu, J. Lu, and L. Ying, A mean field analysis of deep resnet and beyond: Towards provably optimization via overparameterization from depth, arXiv:2003.05508, 2020.
		
		\bibitem{Montanelli1} H. Montanelli and Q. Du, Deep ReLU networks lessen the curse of dimensionality, arXiv:1712.08688, 2017.
		
		\bibitem{Montanelli2} H. Montanelli and H. Yang, Error bounds for deep ReLU networks using the Kolmogorov-Arnold superposition theorem, {\it Neural Networks} \textbf{129} (2020), 1--6.
		
		\bibitem{IdentityMatters} H. Moritz and T. Ma, Identity Matters in Deep Learning, arXiv:1611.04231, 2017.
		
		
		
		\bibitem{Poggio} T. Poggio, H. Mhaskar, L. Rosasco, B. Miranda, and Q. Liao, Why and when can deep-but not shallow-networks avoid the curse of dimensionality: A review, {\it Internat. J. Automat. Comput.} \textbf{14} (2017), 503--519.
		
		\bibitem{Qin} T. Qin, K. Wu, and D. Xiu, Data driven governing equations approximation using deep neural networks, {\it J. Comput. Phys.} \textbf{395} (2019): 620-635.
		
		
		\bibitem{Shen1} Z. Shen, H. Yang, and S. Zhang, Deep network approximation characterized by number of neurons, {\it Commun. Comput. Phys.} \textbf{28} (2020), no. 5, 1768--1811.
		
		\bibitem{Shen2} Z. Shen, H. Yang, and S. Zhang, Deep network with approximation error being reciprocal of width to power of square root of depth, {\it Neural Comput.} \textbf{33} (2021), no. 4, 1005--1036.
		
		\bibitem{Shen3} Z. Shen, H. Yang, and S. Zhang, Optimal approximation rate of ReLU networks in terms of width and depth, arXiv:2103.00502, 2021.
		
		\bibitem{Stein} E. Stein and R. Shakarchi, {\it Fourier Analysis. An introduction}, Princeton University Press, Princeton, NJ, 2003.
		
		
		
		
		\bibitem{Wang} Y. Wang, A mathematical introduction to generative adversarial nets (GAN), arXiv:2009.00169, 2020.
		
		\bibitem{Wedderburn}  J. H. M. Wedderburn, {\it Lectures on Matrices}, Dover, New York, 1964.
		
		\bibitem{ConvergenceDeepReluNet} Y. Xu and
		H. Zhang, Convergence of Deep ReLU Networks, arXiv:2107.12530, 2021.
		
		\bibitem{ConvergenceDeepCNN} Y. Xu and
		H. Zhang, Convergence of Deep Convolutional Neural Networks, arXiv:2109.13542, 2021.
		
		\bibitem{Yarotsky} D. Yarotsky, Error bounds for approximations with deep relu networks, {\it Neural Networks} \textbf{94} (2017), 103--114.
		
		\bibitem{Zaslavsky} T. Zaslavsky, Facing up to arrangements: face-count formulas for partitions of space by hyperplanes, {\it Mem. Amer. Math. Soc.} \textbf{1} (1975), issue 1, no. 154.
		
		\bibitem{Zhang} H. Zhang, X. Gao, J. Unterman, and T. Arodz, Approximation capabilities of neural odes and invertible residual networks, ICML, 2020.
		
		\bibitem{Zhou1} D.X. Zhou, Universality of deep convolutional neural networks, {\it Appl. Comput. Harmon. Anal.} \textbf{48} (2020), no. 2, 787--794.
		
		\bibitem{Zou} D. Zou, P. M. Long, and Q. Gu, On the global convergence of training deep linear ResNets, arXiv:2003.01094, 2020.
		
	\end{thebibliography}
	
}

\end{document}